%% file: main_arxiv.tex
\documentclass[11pt]{article}
\usepackage[left=1in,top=1in,right=1in,bottom=1in,letterpaper]{geometry}

\usepackage[colorlinks, linkcolor=red,filecolor=blue,citecolor=blue,urlcolor=blue]{hyperref}

\usepackage{mathtools} % add 

\usepackage{url}
\usepackage{amsthm,amsmath,amssymb, amscd}
\usepackage{mathrsfs}
\usepackage{amsfonts}
\usepackage[utf8]{inputenc}
\usepackage{subfigure}
\usepackage{graphicx}
\usepackage{bbding}
\everymath{\displaystyle}
\usepackage{indentfirst}

\usepackage{enumerate}
\usepackage{bm,bbm}
\usepackage[shortlabels]{enumitem}
\usepackage{algorithm, algorithmic}
\usepackage{dsfont}
\usepackage{booktabs}
\usepackage{xspace}
\usepackage{pifont}
\usepackage{xcolor}
\usepackage{cleveref}
\usepackage{dsfont}
\usepackage{mathpazo}
\PassOptionsToPackage{compress, square,numbers}{natbib}
\usepackage{natbib}
\numberwithin{equation}{section}
\usepackage{apptools}
\AtAppendix{\counterwithin{thm}{section}}

\input{math_commands}

\input{def}

\newcommand{\cirone}{\text{\ding{172}}}
\newcommand{\cirtwo}{\text{\ding{173}}}
\newcommand{\cirthree}{\text{\ding{174}}}
\newcommand{\cirfour}{\text{\ding{175}}}

\renewcommand{\eqref}[1]{(\ref{#1})}
\usepackage{colortbl}
\definecolor{Ocean}{RGB}{129,194,234}

\title{Deterministic Policy Gradient for Reinforcement Learning with  Continuous Time and State}

\author{
Ziheng Cheng\thanks{University of California, Berkeley. Email: \texttt{ziheng\_cheng@berkeley.edu}}
\and Xin Guo\thanks{University of California, Berkeley. Email: \texttt{xinguo@berkeley.edu}}
\and Yufei Zhang\thanks{Imperial College London. Email: \texttt{yufei.zhang@imperial.ac.uk}}
} 

\date{}

\begin{document}

\maketitle

\begin{abstract}

The theory of continuous-time reinforcement learning (RL) has progressed rapidly in recent years. While the ultimate objective of RL is typically to learn deterministic control policies, most existing continuous-time RL methods rely on stochastic policies. Such approaches often require sampling actions at very high frequencies, and involve computationally expensive expectations over continuous action spaces, resulting in high-variance gradient estimates and slow convergence.

  In this paper, we introduce and develop deterministic policy gradient (DPG) methods for continuous-time RL. We derive a continuous-time policy gradient formula expressed as the expected gradient of an advantage rate function and establish a martingale characterization for both the value function and the advantage rate. These theoretical results provide tractable estimators for deterministic policy gradients in continuous-time RL. Building on this foundation, we propose a model-free continuous-time Deep Deterministic Policy Gradient (CT-DDPG) algorithm that enables stable learning for general reinforcement learning problems with continuous
time-and-state.
  Numerical experiments show that CT-DDPG achieves superior stability and faster convergence compared to existing  stochastic-policy methods, across a wide range of learning tasks with varying time discretizations and noise levels.

\end{abstract}

\section{Introduction}
Deep reinforcement learning (RL) has achieved remarkable success across a wide range of domains, driven by 
theoretical advances and the success of algorithms in discrete-time systems such as Atari, Go, and Large Language Models \citep{mnih2013playing, silver2016mastering,guo2025deepseek}. Algorithms developed for discrete-time environments have demonstrated impressive performance in applications such as game playing, robotics, and large-scale decision systems. Despite these advances, many real-world decision-making problems, including robotic manipulation, autonomous driving, and financial trading,
are inherently continuous in time. In these settings, control actions must evolve smoothly and often at extremely high frequencies, highlighting the need for principled continuous-time reinforcement learning frameworks \citep{wang2020reinforcement}.

\paragraph{Issues with discretization.} 
A common approach for handling continuous-time problems is to apply standard discrete-time RL algorithms after discretizing the time domain. However, this strategy introduces significant challenges
\citep{baird1994reinforcement,doya2000reinforcement,munos2006policy}. In particular, the performance of many RL algorithms is highly sensitive to the discretization step size. As the step size becomes smaller, classical RL algorithms often suffer from severe degradation, including exploding gradient variance, numerical instability, and dramatically slower learning. These issues make naive discretization unsuitable for high-frequency control tasks and motivate the development of algorithms whose behavior remains stable as the discretization step decreases.

Several recent works attempt to address these challenges. For example, discretization-invariant formulations based on the advantage rate function \citep{tallec2019making} have been proposed to stabilize Q-learning methods, and alternative strategies such as action repetition have been introduced to reduce the variance of policy gradient estimators \citep{park2021time}. Empirical studies \citep{treven2023efficient} have also explored the effects of different discretization schemes in continuous-time settings. While these approaches alleviate some sensitivity to the time step, they typically rely on restrictive assumptions, and most notably deterministic dynamics and therefore remain difficult to apply in realistic stochastic  and continuous environments. Moreover, despite substantial empirical evidence of discretization-induced failures, their analytical and theoretical understanding remains missing.

\paragraph{Stochastic policy and its challenges.}
Another line of research formulates continuous-time RL using stochastic policies, often in systems governed by stochastic differential equations. These methods extend classical policy-gradient and actor–critic frameworks to continuous time and establish connections between RL and stochastic control theory.
In particular,
\citep{jia2022policy_eval,jia2022policy_grad} develop  the continuous-time analogue of the discrete-time advantage function, namely the $q$-function, \citep{jia2023q}   proposes a $q$-learning algorithm.  
\citep{giegrich2024convergence,sethi2025entropy} extend natural policy gradient methods to the continuous-time setting, and \citep{zhao2023policy} further generalizes PPO \citep{schulman2017proximal} and TRPO \citep{schulman2015trust} methods to continuous time. 

However, stochastic policies introduce their own set of difficulties. First, they require sampling random actions at very high frequencies, producing irregular control trajectories that may be impractical in real-world systems
\citep{jia2025accuracy}. Second, evaluating the value or Q-function under stochastic policies typically requires integration over continuous action spaces, which is computationally expensive and often approximated using Monte Carlo methods. As a result, deep RL methods based on stochastic policies frequently exhibit unstable training dynamics and slow convergence in practice. (See also experiments in \Cref{sec:exp}).

 Furthermore, from a control perspective, the ultimate objective in many continuous-time problems is to learn a deterministic control policy. Recovering such policies from stochastic formulations can be challenging, particularly when the stochastic policy is represented by a general Gibbs distribution over actions. In fact, existing stochastic-policy approaches are largely tractable only in specialized settings such as linear–quadratic regulator (LQR) problems, where the optimal policy has a simple parametric form.

These limitations naturally raise the following question:
\begin{center}
Can we develop a theoretically grounded and practically efficient reinforcement learning framework that directly learns deterministic policies in continuous-time RL?
\end{center}

\paragraph{Our work.}

In this paper, we address this question by studying deterministic policy gradient (DPG) methods in continuous-time RL. We consider controlled stochastic dynamics described by stochastic differential equations over a finite horizon and develop a model-free framework for deterministic policy optimization. Our analysis establishes the theoretical foundations of deterministic policy gradients in continuous time and connects them with existing stochastic policy formulations. Specifically, our main contributions are

\begin{itemize}[leftmargin=2em]
    \item Continuous-time deterministic policy gradient theory: 
we develop a rigorous mathematical framework for model-free deterministic policy gradient methods in continuous-time RL (\Cref{sec:theory}). In particular, we derive the deterministic policy gradient formula (\Cref{prop:pg})  using the gradient of the advantage rate function and provide a martingale-based characterization that enables model-free estimation (\Cref{thm:martingale_char}). Moreover,
we show   (\Cref{thm:spg_limit}) 
the connection between  the deterministic policy gradient and the stochastic policy gradients when the stochastic policies concentrate around deterministic ones.

\item A stable and efficient deep RL algorithm:
building on the theoretical framework, we propose CT-DDPG, a novel actor–critic algorithm designed for continuous-time environments
(\Cref{sec:algorithms}). The method employs a multi-step temporal-difference objective and demonstrates improved stability with respect to time discretization and stochastic noise.
We establish theoretical analysis regarding robustness of discretization and stochastic noises  (\Cref{prop:var_normal}). For the first time, we provide  theoretical insights into the discretization-induced-failure for deep RL algorithms in continuous and stochastic settings.

\item Numerical experiment: We show,  through extensive experiments  (\Cref{sec:exp}), that CT-DDPG achieves superior stability and faster convergence compared to existing  stochastic policy methods, across a wide range of learning tasks with varying time discretizations and noise levels. 
    In particular, we confirm  that existing discrete/continuous time algorithms lack robustness to time discretization and dynamic noise, while our method exhibits consistently stable performance.

\end{itemize}
Together, our results provide both theoretical insights and practical algorithms for RL in continuous-time systems, bridging the gap between modern deep RL methods and stochastic control formulations.

We would like to emphasize that our work, based on a deterministic policy, represents a significant conceptual departure from earlier studies on stochastic policies.  Our framework and analysis demonstrate, for the first time, how to design efficient model-free  algorithms to learn the deterministic policy {\it directly} in continuous-time RL problems. 
The proposed deterministic policy gradient is distinct from the stochastic policy gradient formulas in \citep{jia2023q} (\Cref{rmk:spg_degenerate}). Moreover, the martingale characterization of the advantage rate function requires only local exploration around the current action (\Cref{rmk:local_exploration}), and    avoids the costly integrations over continuous action spaces required in the stochastic policy approach (\Cref{rmk:DPGvsSPG}), thereby enabling sample-efficient off-policy learning. 
Leveraging deterministic policies, we are able to enforce the Bellman equation via a simple reparameterization trick, leading to a stable deep RL implementation.

\paragraph{Notations.}  
We denote by 
$ C([0, T]\times \sR^n) $
(reps.~$C_b([0, T]\times \sR^n) $)
the space of continuous (resp.~bounded) functions $ u : [0, T] \times \sR^n  \to \mathbb{R} $, 
and  by  
 $ C^{1,2}([0, T]\times \sR^n) $
the space of   functions  
$ u \in C([0, T]\times \sR^n)  $   that are once continuously differentiable in time and twice continuously differentiable in space, and there exists a constant $C\ge 0$
such that for all $(t,x)\in [0,T]\times \sR^n$, $|u(t,x)|+|\partial_t u(t,x)|\le C(1+|x|^2), |\partial_x u(t,x)|\le C(|1+|x|)$ and $  |\partial^2_{xx} u(t,x)| \le C$. 
We denote by  
$C^{i+\alpha/2, j+\alpha}_b([0,T]\times \sR^{n})$, $i,j\in \sN\cup\{0\}$,   the space of    functions 
$u:[0,T]\times \sR^{n}\to \sR$  
such that 
$
\|u\|_{{i+\alpha/2,j+\alpha }} \coloneqq \sum_{\ell= 0}^{i} \|\partial^{\ell}_t u\|_{0 } 
+ \sum_{\ell= 0}^{j}  \|\partial^\ell_x u\|_{0} 
+ [\partial^{i}_t u]_{\alpha/2 } 
+[\partial^{j}_x u]_{\alpha } 
<\infty
$, where
$\|u\|_{0}\coloneqq \sup_{(t,x) \in [0,T]\times \sR^n}
|u(t,x)|$
and 
$
[u]_{\alpha} \coloneqq  \sup_{(t,x),(t,x')\in [0,T]\times \sR^n}
\frac{|u(t,x)-u(t',x')|}{(|t-t'|^{1/2}+|x-x'|)^\alpha } 
$. 
For each metric space $E$,
we denote by 
 $\mathcal{P}(E)$   the space  of   probability measures on $E$, equipped with the topology of weak convergence.
Finally,
given $f,g:(0,\infty)\to \mathbb R$,
we write $f(x)=\mathcal{O}(g(x))$ 
if 
$\limsup_{x\to 0}|f(x)|/|g(x)|<\infty$,
write   $f(x)=\Omega(g(x))$ if 
 $\liminf_{x\to 0}|f(x)|/|g(x)|>0$, and   write $f(x)=\Theta(g(x))$ if $f(x)=\mathcal{O}(g(x))$ and $f(x)=\Omega(g(x))$.

\section{Problem Formulation of Continuous RL }
  
This section formulates the   RL  problem with continuous time and space, where  the agent   learns  an optimal parametrized policy to  control a stochastic differential equation with unknown coefficients  to maximize a reward functional over a finite time horizon. {The extension to the infinite time horizon is  similar, and can be developed with appropriate modifications.}

Let $T>0$  be a given terminal time and 
$(\Omega, \mathcal F,\sP)$   be a complete probability space which supports  an $n$-dimensional  Brownian motion $W$
and an independent square-integrable random variable $\xi_0$.
   We denote by   $\mathbb{F}=(\mathcal F_t)_{t\ge 0}$  the filtration generated by $W$ and $\xi_0$ augmented by
the $\mathbb P$-null sets. Let 
 $\mathcal{A}\subseteq\R^d$ be a measurable set representing the agent's action space,
 and let $
 \mathcal U$ be a space of all 
 $\mathbb F$-adapted square-integrable processes 
 $\alpha:[0,T]\times \Omega\to \mathcal{A}$, representing the agent's admissible control space.
For each $\alpha\in \mathcal U$,
consider the associated  state dynamics:
\begin{equation}
\label{eq:control_dynamics}
    \dif X_t^{{\alpha}}=b(t,X_t^{\alpha},\alpha_t)\dif t+\sigma(t,X_t^{\alpha},\alpha_t)\dif W_t,
    \; t\in[0,T];
    \quad  X_0^{\alpha}=\xi_0,  
\end{equation}
where   $b:[0,T]\times\R^n\times \R^d\to\R^n$, $\sigma:[0,T]\times\R^n\times \R^d\to\R^{n\times n}$ are continuous    functions such that \eqref{eq:control_dynamics} admits a unique square-integrable solution $X^\alpha$. The agent aims to maximize the following reward functional 
\begin{equation}
   \label{eq:control_objective} \E\left[\int_0^Te^{-\beta t}r(t,X_t^{\alpha},\alpha_t)\dif t+e^{-\beta T}g(X_T^{\alpha})\right]
\end{equation}
over all admissible controls $\alpha \in \mathcal U$,  
where $\beta\ge 0$ is a discount factor, and the running reward $r:[0,T]\times\R^n\times \R^d\to \R$ and the terminal reward $g:\R^n\to \R$ are continuous functions with at most quadratic growth.

It is well-known  that 
under mild regularity conditions\footnotemark, 
it suffices to optimize 
\eqref{eq:control_objective} over     control processes of the form 
$\alpha^\mu_t=\mu(t,X^\mu_t)$ for all $t\in [0,T]$,
where 
     $\mu:[0,T]\times \sR^n\to \mathcal{A}$ is a Markovian control and $X^\mu$ is a solution to the following controlled   dynamics:
\begin{equation}
\label{eq:state_mu}
    \dif X^\mu_t=b(t,X^\mu_t,\mu(t,X^\mu_t))\dif t+\sigma(t,X^\mu_t,\mu(t,X^\mu_t))\dif W_t, 
    \;t\in [0,T];
    \quad X^\mu_0=\zeta_0. 
\end{equation} 
Consequently, the goal of the agent is to maximize  the following objective 
\begin{equation}
\label{eq:cost_mu}
\E\left[\int_0^Te^{-\beta t}r(t,X_t^\mu,\mu(t,X^\mu_t))\dif t+e^{-\beta T}g(X_T^\mu)\right]
\end{equation}
over all admissible Markovian policies $\mu$. 

\footnotetext{
Sufficient conditions are that the coefficients satisfy suitable convexity conditions with respect to the action 
\cite[Theorem 2.10]{nicole1987compactification}, or that the diffusion coefficient is non-degenerate \cite{krylov2008controlled}.
}

In an RL framework, the agent does not have access to the coefficients $b$, $\sigma$, $r$ and $g$. Instead, the agent directly interacts with  \eqref{eq:state_mu} with different actions, 
and refines her strategy based on  observed state and reward trajectories. 
Following the classical control framework,   throughout the paper we restrict  admissible policies to be Markovian policies $\mu:[0,T]\times \sR^n\to \mathcal A$ to optimize \eqref{eq:cost_mu}.  
These policies are referred to as deterministic policies in the RL literature, as they map time and state directly to a given action.

   To solve this RL problem,  
   %with high-dimensional state and action spaces,   
   one typically restricts   the optimization problem \eqref{eq:cost_mu} to   a sufficiently rich class of parameterized policies. More precisely, 
given a class of policies
$\{\mu_\phi:[0,T]\times \sR^n\to \mathcal{A} \mid  \phi\in \R^k\} $  parameterized by $\phi$, we consider the  following maximization problem:
\begin{equation}
\label{eq:cost_phi}
\max_{\phi\in \R^k} J(\phi),
\quad \textnormal{with}
\quad 
J(\phi)  \coloneqq \E\left[\int_0^Te^{-\beta t}r(t,X_t^\phi,\mu_\phi(t,X^\phi_t))\dif t+e^{-\beta T}g(X_T^\phi)\right],
\end{equation}
where  
 $X^\phi$ denotes   the state process 
 \eqref{eq:state_mu} controlled by $\mu_\phi$. 
Throughout this paper, we  
 impose the following regularity conditions on 
   the  model coefficients and the policy class.

\begin{Assumption}
\label{assum:wp}

    There exists $C\ge 0$
such that 
for all $t\in [0,T]$, $a, a'\in \mathcal{A}$   and $x,x'\in \sR^n$,
\begin{align*}
|b(t,x,a)-b(t,x',a')|+|\sigma(t,x,a)-\sigma(t,x',a')|
&\le C(|x-x'|+|a-a'|),
\\
|b(t,0,0)|+
|\sigma(t,0,0)|
\le C,
\quad 
|r(t,x,a)|+|g(x)|&\le C(1+|x|^2+|a|^2).
\end{align*}
There exists a locally bounded function $\rho_1:[0,\infty)\to [0,\infty) $  such that for all $\phi\in  \sR^k$, $t\in [0,T]$, and 
$x,x'\in \sR^n$,
$|\mu_\phi(t,x) -\mu_\phi(t,x') | \le \rho_1(|\phi|) |x-x'|$
 and 
 $|\mu_\phi(t,0)|\le \rho_1(|\phi|)$.

 \end{Assumption}
  
\Cref{assum:wp} holds for all policies parameterized by feedforward neural networks and transformers with Lipschitz activations. It ensures that the state dynamics and the objective function are well defined for any $\phi\in \R^k$.

\section{Main Theoretical Results}\label{sec:theory}
 We will first characterize the gradient of the objective functional \eqref{eq:cost_phi} with respect to the policy parameter $\phi$, using a  continuous-time analogue of the discrete-time advantage function. 
We will then derive a martingale characterization of this continuous-time advantage   function and value function, which serves as the foundation of our algorithm design under deterministic policies.

\subsection{Deterministic Policy Gradient (DPG) Formula}

We first  introduce a dynamic version of  
the objective function $J(\phi)$. For each $(t,x)\in [0,T]\times \sR^n$, define the value function
\begin{align}
\label{eq:value_t_x}
V^\phi(t,x ) \coloneqq \mathbb E \left[
\int_t^T 
e^{-\beta (s-t)}
r(s,X_s^\phi,\mu_\phi(s,X^\phi_s))\dif s+e^{-\beta (T-s)}g(X_T^\phi)
\,\bigg\vert\, X_t^\phi=x
\right].
\end{align}
Note that $J(\phi)=\E[V^\phi(0,\xi_0)]$.
We additionally impose the following differentiability condition on the model parameters and policies with respect to the parameter.
\begin{Assumption}
\label{assum:differentiability}
Let $\Sigma\coloneqq \sigma\sigma^\top$.
The functions 
$\partial_a b,\partial_a \Sigma, \partial_a r$,
and $\partial_\phi \mu $
exist and are continuous. 
    There exists a locally bounded function $\rho_2:[0,\infty)\to [0,\infty) $  such that for all $\phi\in \sR^k$
    and $(t,x)\in [0,T]\times \sR^n$,
$$
\frac{|\partial_\phi b(t,x,\mu_\phi(t,x))|}{1+|x|}
+
\frac{|\partial_\phi \Sigma(t,x,\mu_\phi(t,x))|
+ |\partial_\phi r(t,x,\mu_\phi(t,x))|}{ 1+|x|^2}\le \rho_2(|\phi|).
$$
  $V^\phi\in C^{1,2}([0,T]\times \sR^n)$ for all $\phi\in \sR^k$.

\end{Assumption}
 
Under \Cref{assum:wp}, we have the Feyman-Kac formula for the value function. That is, for any given $\phi\in \sR^k$,
$V^\phi\in C^{1,2}([0, T]  \times \sR^n)$ satisfies 
the following linear Bellman equation:
for all $(t,x)\in [0,T]\times \sR^n$, 
\begin{align} 
\label{eq:pde_v}
\begin{split} 
 &  \mathcal L [ V^\phi](t,x, \mu_\phi(t,x)) +r(t,x,\mu_\phi(t,x)) =0,
\quad 
  V^\phi(T,x) = g(x),
  \end{split}
\end{align}
where  
$\mathcal L$   is the generator of \eqref{eq:state_mu}
    such that 
for all $\varphi\in C^{1,2}([0, T] \times \sR^n) $,
\begin{align}
\label{eq:generator_L}
\begin{split}
 \mathcal L[ \varphi](t,x, a)
 &\coloneqq  \partial_t \varphi(t,x)
     -\beta   \varphi(t,x)
     +  b(t,x, a)^\top \partial_x \varphi(t,x) +\frac{1}{2} \textrm{Tr}( \Sigma (t,x,a)     
     \partial^2_{xx} \varphi(t,x)).
     \end{split}
\end{align}

Moreover, by 
quantifying the difference between the value functions corresponding to two policies, along with Vitali’s convergence theorem, one can obtain
the  following   DPG formula for the continuous RL problem.  
A similar  formula has been   established in \citep{gobet2005sensitivity}   under a different set of conditions.

\begin{prop}
\label{prop:pg}
 Suppose \Cref{assum:wp,assum:differentiability} hold.  
 For all   $(t,x)\in [0,T]\times \sR^n$
 and $ \phi \in \sR^k$,
 \begin{align*}
 \begin{split}
&  \partial_\phi V^{\phi}(t,x)
   =  
     \mathbb E \left[ \int_t^T
   e^{-\beta(s-t)} 
   \partial_\phi \mu_\phi(s,X^\phi_s)^\top 
    \partial_a A^\phi(s, X^{  \phi }_s,\mu_\phi(s,X^\phi_s)) 
   \dif s\,\bigg\vert\, X^\phi_t=x \right],
 \end{split}
 \end{align*}
 where 
 $A^\phi(t,x,a)\coloneqq  \mathcal L[ V^\phi](t,x, a)  +  r(t, x,   a)$.
\end{prop}

\begin{rmk} 

\Cref{prop:pg} is analogous to the DPG 
formula for discrete-time Markov decision processes \citep{silver2014deterministic}.
The function 
$A^\phi$
plays the   role of advantage function   used in  discrete-time DPG, 
and has been   referred to as the advantage \textit{rate} function in \citep{zhao2023policy}. 
Recall that 
the advantage function in discrete-time RL is  defined as the difference between the $Q$-function and value function for a given policy $\mu$ \cite{sutton1999policy,silver2014deterministic}.
It 
measures how much better a specific action $a$
 is compared to the average action in any given state $s$
 under some policy  $\mu$. It reduces variance in policy gradient methods, with clearer and more stable learning signals.

To see this connection, 
assume $\beta=0$, and 
for any     $N\in \mathbb N$, consider the  discrete-time version of \eqref{eq:cost_phi}:
\begin{equation}
\label{eq:cost_phi_delta_t}
J_{\Delta t}(\phi)  \coloneqq \E\left[\sum_{i=0}^{N-1} r(t_i,X_{t_i}^{\Delta t,\phi},\mu_\phi(t_i,X^{\Delta t,\phi}_{t_i}))\Delta t+ g(X_T^{\Delta t,\phi})\right],
\end{equation}
where 
$\Delta t= {T}/{N}$, 
$t_i=i\Delta t$, and 
$X^{\Delta t,\phi}$ satisfies the following time-discretization  of   \eqref{eq:state_mu}:
$$
X_{t_{i+1}}^{\Delta t,\phi}=
X_{t_{i}}^{\Delta t,\phi}+
b(t_i,X_{t_{i}}^{\Delta t,\phi},\mu_\phi(t_i,X_{t_{i}}^{\Delta t,\phi}))\Delta t+\sigma(t_i,X_{t_{i}}^{\Delta t,\phi},\mu_\phi(t_i,X_{t_{i}}^{\Delta t,\phi}))\sqrt{\Delta t} \omega_{t_i},
$$
and $(\omega_{t_i})_{i=0}^{N-1}$ are independent standard normal random variables. By the discrete-time DPG formula in \citep[Equation (9)]{silver2014deterministic}, 
\begin{equation}
\label{eq:dpg_discrete}
        \partial_\phi J_{\Delta t}(\phi) =
        \mathbb E\left[\sum_{i=0}^{N-1} \partial_\phi\mu_\phi(t_i,X_{t_{i}}^{\Delta t,\phi} )^\top
\partial_a  
A^{\Delta t, \phi}(t_i, X_{t_{i}}^{\Delta t,\phi},\mu_\phi(t_i,X_{t_{i}}^{\Delta t,\phi}))\Delta t
 \right],
\end{equation}
where $A^{\Delta t, \phi} (t,x,a)\coloneqq     \frac{Q^{\Delta t, \phi}(t,x,a)-V^{\Delta t, \phi}(t,x)}{\Delta t}$ is the (discrete-time) advantage function for \eqref{eq:cost_phi_delta_t}
normalized with the time stepsize. 
As $N\to \infty$, $A^{\Delta t, \phi}$ converges to  $A^\phi$, as shown in  \citep{jia2023q}. Sending $\Delta t\to 0$ in \eqref{eq:dpg_discrete} yields the continuous-time DPG in \Cref{prop:pg}.
\end{rmk}

\begin{rmk}
\label{rmk:spg_degenerate}
\Cref{prop:pg} is distinct from the stochastic policy gradient formulas for continuous-time RL given in \cite[Sectioin 3.2]{jia2022policy_grad}. Indeed, the stochastic policy gradient formulas there  involve  the gradient of the log-density of the stochastic policy. These formulas cannot be applied to a deterministic policy, whose action distributions are Dirac measures and therefore do not admit a density. See \Cref{sec:dpg_spg} for more details.

\end{rmk}

\subsection{Martingale Characterization of Continuous-Time Advantage Rate Function}

\Cref{prop:pg} suggests that implementing DPG 
for continuous RL 
relies on   learning the advantage rate function 
$A^\phi$ in a neighborhood of the   policy $\mu_\phi$. 
The following theorem establishes a martingale criterion to jointly characterize the value function and the advantage rate function. It provides the theoretical foundation for designing a model-free actor–critic learning algorithm for continuous RL problems (see Algorithm \ref{alg:ddpg_seq}).  
 
\begin{thm}
\label{thm:martingale_char}
 Suppose   \Cref{assum:wp,assum:differentiability} hold. Let  $\phi\in \sR^k$,    $\hat{V}\in {C}^{1,2}([0,T]\times\R^n)$ and $\hat{q}\in {C}([0,T]\times\R^n \times \mathcal{A})$ satisfy
    the following conditions
    for all $(t,x)\in[0,T]\times\R^n$: 
\begin{equation}\label{eq:hjb_condition}
    \hat{V}(T,x)=g(x),
    \quad \hat{q}(t,x,   \mu_\phi(t,x ))=0, 
    \end{equation}
and 
for  a given  neighborhood 
    $\mathcal O_{\mu_\phi(t,x)} \subset \mathcal A$
       of $\mu_\phi(t,x)$, it holds  for all $a\in \mathcal O_{\mu_\phi (t,x)}$ that, there exists a square-integrable $\mathcal A$-valued adapted 
  process $(\alpha_s)_{s\ge t}$  with   $\lim_{s\searrow t}\alpha_s =a$ almost surely and
\begin{equation}\label{eq:martingale}
     \left(  e^{-\beta (s-t)}\hat{V}(s,X_s^{t,x, a })+\int_t^s e^{-\beta (u-t)}
             (r- \hat{q}) (u,X_u^{t,x,a  },   \alpha_u) 
            \dif  u
           \right)_{s\in [t,T]} 
    \end{equation}
    is an $\mathbb{F}$-martingale, where 
   $X^{t,x, a }$ satisfies  for all $s\in [t,T]$,
   \begin{align}
\label{eq:martingale_state_mv}
\begin{split}
 \dif X^{t, x, a}_s &=b(s,X^{t, x, a}_s ,\alpha_s)\dif s+\sigma (s,X^{t, x,  a}_s,  \alpha_s)\dif W_s,
  \quad X^{t, x, a}_t=x. 
 \end{split} 
\end{align} 
  Then 
  $\hat V(t,x)=V^\phi(t,x)$
  and 
  $\hat{q}(t,x,a )=A^\phi(t,x,a)$
  for all $(t,x,a )\in [0,T]\times\R^n\times \mathcal O_{\mu_\phi(t,x)}$.

\end{thm}

\Cref{thm:martingale_char} establishes   conditions ensuring that the functions $\hat V$ and $\hat q$ coincide with the value function and the advantage rate function of  a given policy $\mu_\phi$, respectively.
 Equation \eqref{eq:hjb_condition}
requires that $\hat V$ agrees with the terminal condition $h$ at time $T$,
and the function $\hat q$ 
satisfies the linear Bellman equation 
\eqref{eq:pde_v} 
as the true advantage rate $A^\phi$.
The martingale constraint \eqref{eq:martingale} ensures $\hat q$ is the advantage rate function associated with $\hat V$,  for all actions in a neighborhood of the policy $\mu_\phi$.

We highlight two features that distinguish the martingale characterization of the advantage rate function for deterministic policies from existing characterizations for stochastic policies \citep{jia2022policy_grad, jia2023q,zhao2023policy}. 
\begin{rmk}[Local exploration]
  \label{rmk:local_exploration}  
\Cref{thm:martingale_char} only requires the martingale condition \eqref{eq:martingale} to hold for state processes initialized with actions  $a \in \mathcal O_{\mu_\phi(t,x)}$,
which represent local exploration around the current action.
Such local exploration is sufficient because the DPG formula in \Cref{prop:pg}   depends only on the local behavior of the advantage rate function near the current action.  The resulting policy gradient updates are therefore more efficient than those based on stochastic policies,  which requires generating (suboptimal) exploratory actions across the entire action space to learn the advantage function for all possible actions.
 
In practice, one can use an exploration policy to generate these exploratory  actions around the current actions, which are then employed to learn the gradient of the target deterministic  policy. This parallels the central role of off-policy algorithms in discrete-time DPG methods \citep{lillicrap2015continuous,haarnoja2018soft}.
\end{rmk}

\begin{rmk}[Simplified Bellman equation] \label{rmk:DPGvsSPG}

The Bellman condition \eqref{eq:hjb_condition} only involves  pointwise evaluations of the deterministic policy $\mu_\phi$, and hence can be easily implemented using a simple reparameterization (see \eqref{eq:repara_q}). 
This leads to more efficient and stable learning of the advantage rate function compared with the stochastic policy approach, whose corresponding Bellman condition requires costly integrations over the continuous action space.

To see this, recall that \citep{jia2022policy_grad, jia2023q,zhao2023policy} study continuous-time RL with stochastic policies $\pi:[0,T]\times \sR^d\to \mathcal P(\mathcal A)$ by incorporating an additional entropy   term into the objective. \citep{jia2023q} characterizes the advantage rate function analogously to \Cref{thm:martingale_char}, replacing the Bellman condition \eqref{eq:hjb_condition} with
\begin{equation}\label{eq:consistency_stochastic}
\int_{\mathcal A}
\left(\hat q(t,x,a)-\gamma\log\pi(a|t,x)\right) \pi(  a|t,x) \dif a= 0,
\quad \forall (t,x)\in [0,T]\times \sR^n,
\end{equation}
where $\gamma>0$ is the entropy regularization coefficient, and requiring the martingale constraint \eqref{eq:martingale} to hold for all state dynamics starting at state $x$ at time $t$, with actions sampled randomly from $\pi$ at any time partition of $[0,T]$.
Implementing the condition \eqref{eq:consistency_stochastic} 
requires computing the integral over the   action space,
which is typically estimated by Monte-Carlo methods using  random actions from the policy $\pi$. This makes policy evaluation substantially noisier and more expensive, particularly with high-dimensional action spaces and general function approximators for   $\hat q$ and $\pi$.
As a result,   RL with stochastic policies suffers from training instability and slow convergence, as demonstrated  in experiments in   \Cref{sec:exp}. 

\end{rmk} 
\subsection{Connecting DPG  with Stochastic Policy Gradient}
\label{sec:dpg_spg}

In this section, we show analytically the connection between the DPG and the stochastic policy gradient (SPG).

We  
fix a deterministic policy $\mu_\phi:[0,T]\times \sR^n\to \mathcal A$,
and 
impose  the following regularity conditions on the coefficients and the   policy class.
In the sequel, 
for any $\alpha,\beta\in (0,1]$
and $f: [0,T]\times \sR^d\to \sR^m$,
we define the   quantity 
$$
\|f\|_{\alpha,\beta}
\coloneqq 
\sup_{t,t'\in [0,T],x,x'\in \sR^d}
\left(|f(t,x)|
+ \frac{|f(t,x)-f(t',x)|}{|t-t'|^\alpha}
+ \frac{|f(t,x)-f(t,x')|}{|x-x'|^\beta}\right).
$$

\begin{Assumption}
\label{assum:regularity_limit}
   \begin{enumerate}[(1)]
       \item $\mathcal A=\mathbb R^d$.  $\partial_a b$, $ \partial_a\sigma$ and $\partial_a r$ exist and are continuous.
    There exists  $\alpha\in (0,1)$ such that
    \begin{align}
    \|b\|_{{\alpha,1,1}}
    +\|\sigma\|_{\alpha,1,1}
    +\|r\|_{\alpha,1,1}
    <\infty,
    \\
    \sup_{t\in[0,T]}\left(\|\partial_a b(t,\cdot)\|_{ \alpha,1}
    +\|\partial_a \sigma(t,\cdot) \|_{\alpha,1}
    +\|\partial_a r(t,\cdot)\|_{\alpha,1}\right)
<\infty,
\end{align}
  and 
     $g\in C^{2+\alpha}_b(\sR^n)$.
     There exists $
     \kappa>0$ such that 
  $\xi^\top\sigma\sigma^\top(t,x,a) \xi\ge \kappa |\xi|^2$ for all $\xi\in \sR^n$. 
  \item 
  $\partial_\phi \mu$ exists and is continuous. 
  There exists $\alpha\in (0,1)$, and
  a locally bounded function $\rho:[0,\infty)\to [0,\infty) $  such that for all $\phi\in  \sR^k$, 
  \begin{equation}
  \label{eq:policy_holder}
  \|\mu_{\phi}\|_{\alpha,1}
  +\sup_{t\in [0,T]}\|\partial_\phi\mu_{\phi}(t,\cdot)\|_{\alpha}\le \rho(|\phi|).
  \end{equation}

   \end{enumerate} 
\end{Assumption}

 \begin{rmk}
      Assumption \ref{assum:regularity_limit} implies 
Assumptions \ref{assum:wp} and \ref{assum:differentiability},  
hence the DPG formula in  \Cref{prop:pg} holds. In particular,  
by the non-degeneracy   of the diffusion coefficient and 
  Schauder’s estimate \cite[Theorem 9.2.3]{krylov1996lectures}, 
  for each $\phi\in \sR^k$, the value function $V^\phi$ in \eqref{eq:value_t_x} is  a classical solution to the PDE \eqref{eq:pde_v} 
and in the   space $C^{1+\alpha/2, 2+\alpha}_b([0,T]\times \sR^n)$. 
 \end{rmk}

We now consider a family of stochastic policies  $\pi_{\phi, \tau}:[0,T]\times \sR^n\to \mathcal P(\mathcal A)$,  indexed by  $\tau>0$, 
which   
converge to  $\mu_\phi$ as $\tau\to 0$. 
These policies can arise by adding independent exploration noise to the policy $\mu_\phi$, and $\tau$
  specifies the magnitude of this exploration    (see \Cref{ex:exploration_policy}).

For each stochastic policy $\pi_{\phi,\tau}$, 
we consider the associated state process    $\tilde X^{\phi,\tau}$ governed by the  following dynamics:
\begin{equation}
\label{eq:aggregated_sde}
\dif   X_t = \tilde b^{\pi_{\phi,\tau}}(t, X_t )\dif t + \tilde \sigma^{\pi_{\phi,\tau}}(t, X_t )\dif W_t,
\quad t\in [0,T]; 
\quad   X_0 =\xi_0,
\end{equation}
where  the coefficients  $\tilde b^{\pi_{\phi,\tau}}:[0,T]\times \sR^n\to \sR^n $ and $\tilde \sigma^{\pi_{\phi,\tau}}: [0,T]\times \sR^n\to \sR^{n\times n} $    are defined  by
\begin{equation}
	\label{eq:coef_aggregate}
	\begin{aligned}
		& \tilde b^{\pi_{\phi,\tau}}(t,x) \coloneqq \int_{\mathcal A} b(t,x,a)\pi_{\phi,\tau}( \dif a | t,x)  ,\quad \tilde \sigma^{\pi_{\phi,\tau}}(t,x)    \coloneqq \sqrt{\int_{\mathcal A} \Sigma(t,x,a)\pi_{\phi,\tau}(\dif a | t,x) },
	\end{aligned}
\end{equation}
where $\Sigma=\sigma\sigma^\top$
and  $(\cdot)^{1/2}$
is the   principal square root of positive semidefinite matrices.
  The    dynamics
\eqref{eq:aggregated_sde}
has been proposed in  \cite{wang2020reinforcement}, and 
 referred to as the exploratory dynamics. It models interacting with the system \eqref{eq:control_dynamics} by   sampling random actions according
to the   policy $\pi_{\phi,\tau}$; see \cite{jia2025accuracy} for   more details on the connection between    the exploratory dynamics and controlling \eqref{eq:control_dynamics} with
random actions. 
Suppose that \eqref{eq:aggregated_sde}
admits a  unique strong solution $\tilde X^{\phi,\tau}$, 
we  define  the associated value function 
\begin{align}
\label{eq:value_aggregate}
 \tilde V^{\phi,\tau} (t,x ) \coloneqq \mathbb E \left[
\int_t^T 
e^{-\beta (s-t)}
\tilde r^{\pi_{\phi,\tau}}(s,\tilde X_s^{\phi,\tau})\dif s+e^{-\beta (T-s)}g(\tilde X_T^{\phi,\tau})
\,\bigg\vert\, \tilde X^{\phi,\tau}_t=x
\right],  
\end{align}
where 
$\tilde r^{\pi_{\phi,\tau}} :[0,T]\times \sR^n\to \sR$ is defined by 
\begin{equation}
	\label{eq:reward_aggregate}
	\begin{aligned}
		& \tilde r^{\pi_{\phi,\tau}}(t,x) \coloneqq \int_{\mathcal A} r(t,x,a)\pi_{\phi,\tau}(\dif a | t,x).
	\end{aligned}
\end{equation}
 
  In the sequel,    we assume that  the stochastic policy $ \pi_{\phi,\tau}(\dif a|t,x) $ admits a sufficiently  regular density and asymptotically concentrates on $\mu_\phi(t,x)$ as follows.
 
\begin{Assumption}
\label{assum:stochastic_policy}
For all $\phi\in \sR^k$, $\tau>0$ and $(t,x)\in [0,T]\times \sR^n$,
 $\pi_{\phi,\tau}(\dif a|t,x)$  is absolutely continuous 
 with respect to the Lebesgue measure   $\dif a$, and 
  the density $\pi_{\phi,\tau}(t,x,a)$ 
  is of the form:
\begin{align}
\label{eq:density}
  \pi_{\phi,\tau}(t,x,a)=\chi_\tau(a-\mu_\phi(t,x)),
  \end{align}
  where $\chi_\tau:\sR^d\to [0,\infty)$
 is   an integrable function such that 
 \begin{enumerate}[(1)]
        \item \label{item:chi_regularity}
       $\chi_\tau$ is continuously differentiable on its support,
the derivative $\chi'_\tau $ is integrable,  and  
       $\lim_{R\to \infty}\int_{\partial B_R}\chi_\tau    \dif S=0$,
       where $B_R $ is the sphere of radius $R$,  
          and $\dif S$ is the  surface measure   on the boundary $\partial B_R$.
       \item 
      There exists   $C_\chi>0$ such that for all $\tau>0$ and $u,v\in \sR^d$,
            \begin{align}
       \sup_{f\in \operatorname{BL}_{1}(\sR^d)}    
       \left(\int_{\sR^d}f(a)\chi_\tau(a-u)\dif a-\int_{\sR^d}f(a)\chi_\tau(a-v)\dif a\right)&\le C_\chi|u-v|,
       \label{eq:lips_weak}
       \\
       \lim_{\tau \searrow   0}\sup_{f\in \operatorname{BL}_{1}(\sR^d)}    
       \left(\int_{\sR^d}f(a)\chi_\tau(a-u)\dif a-f(u)\right)&=0,
       \label{eq:conv_weak}
       \end{align}
      where  $\operatorname{BL}_{1}(\sR^d)$
      consists of  all  functions $f:\sR^d\to \sR$ such that 
      for all 
      $a,a'\in \sR^d$,
      $|f(a)|\le 1$ and 
      $|f(a)-f(a')|\le |a-a'|$. 
      
    \end{enumerate}

\end{Assumption}

\begin{rmk}
Assumption \ref{assum:stochastic_policy}
is analogous to the  assumptions  imposed in \cite{silver2014deterministic} to analyze the relationship between  DPG  and SPG  for discrete-time Markov decision processes with compact state spaces.
Condition 
\eqref{eq:density}   allows expressing the gradient   $\partial_\phi\tilde V^{\phi,\tau}$
using       $\partial_\phi \mu_\phi$; see \Cref{lemma:spg}.
Conditions \eqref{eq:lips_weak}
and \eqref{eq:conv_weak}
are  equivalent to 
$ \dif_{\rm BL}(\chi_\tau(a-u)\dif a, \chi_\tau(a-v)\dif a)\le C_\chi|u-v|$
and 
 $\lim_{\tau\to 0}\dif_{\rm BL}(\chi_\tau(a-u)\dif a,\delta_u)=0$,
 where 
 $\dif_{\rm BL}$ denotes  the Bounded Lipschitz metric  on the space $\mathcal P(\sR^d)$ (see~\eqref{eq:BL_metric}),
and metrizes the topology of weak convergence on $\mathcal P(\sR^d)$. 
\end{rmk}

The following example shows that Assumption \ref{assum:stochastic_policy} is satisfied by stochastic policies induced by adding independent exploration noises to the deterministic policy
 (see e.g., \cite{szpruch2024optimal} for the case of Gaussian exploration noise).

\begin{example}
\label{ex:exploration_policy}
For each $\tau>0$, 
let    $\pi_{\phi,\tau}(\dif a|t,x)= \sP_{\mu_\phi(t,x)+\tau \xi}$,
where $\xi$ is an $\sR^d$-valued random variable with a  density $h$,
and $\sP_U$ is the law of the random variable $U$. 
Then $\{\pi_{\phi,\tau}\}_{\tau >0}$ satisfies 
\Cref{assum:stochastic_policy} provided that $h$
 enjoys appropriate regularity properties.

\end{example}

We are now ready to establish the consistency between the DPG and SPG formulas, when the stochastic policies asymptotically concentrate on the deterministic policy.

 \begin{thm}\label{thm:spg_limit}
      Suppose \Cref{assum:regularity_limit,assum:stochastic_policy} hold.
      For all   $(t,x)\in [0,T]\times \sR^n$ and $ \phi \in \sR^k$,
      $$
      \partial_\phi V^\phi(t,x)=
      \lim_{\tau\to 0}\partial_\phi \tilde V^{\phi,\tau}(t,x).
      $$
 \end{thm}

\Cref{thm:spg_limit} extends    \cite[Theorem 2]{silver2014deterministic}, which was established for discrete-time Markov decision processes, to the present continuous-time setting. The proof requires  delicate PDE arguments for the convergence of the associated value functions and their derivatives.

\section{Algorithm Design and Analysis }\label{sec:algorithms}

\subsection{Algorithm Design}

Given the martingale characterization  (\Cref{thm:martingale_char}), we now discuss the implementation details in a continuous-time RL framework via  deep neural networks. We use $V_\theta,q_\psi,\mu_\phi$ to denote the neural networks for value, advantage rate function and policy, respectively.

\paragraph{Martingale loss.}
To ensure the martingale condition \eqref{eq:martingale}, let $M_t=e^{-\beta t}V_\theta(t,x_t)+\int_0^te^{-\beta s}
            [r(s,x_s,a_s)-q_\psi(s,x_s,a_s)]\dif s.$
We adopt the following martingale orthogonality conditions (also known as generalized moment method) 
     $\E\left[\int_0^T\zeta_t\dif M_t\right]=0$,
where $\bm{\zeta}=(\zeta_t)_{[0,T]}$ is any test function. This is both necessary and sufficient  to ensure the martingale condition for    all $\mathbb{F}$-adapted and square-integrable processes $\bm{\zeta}$ \citep{jia2022policy_eval,bo2025optimal}.

 %   \eqref{eq:martingale} is a martingale if and only if \eqref{eq:martingale_ortho} holds for any $\{\mathcal{F}_t\}_{t\geq 0}$-adapted and square-integrable process $\bm{\zeta}$.

In theory, one should consider all possible test functions, which leads to   infinitely many equations.
For practical implementation, however, it suffices to select a finite number of test functions with special structures.
A natural choice is to set $\zeta_t=\partial_\theta V_\theta(t,x_t)$ or $\zeta_t=\partial_\psi q_\psi(t,x_t,a_t)$, in which case %\eqref{eq:martingale_ortho}
the marginal orthogonality condition 
becomes a vector-valued condition.
The classic stochastic approximation method \citep{robbins1951stochastic} can be applied to solve the equation:
\begin{equation*}
    \theta\leftarrow\theta-\eta\partial_\theta V_\theta(t,x_t) \cdot\Big(V_\theta(t,x_t)-\int_t^{t+\delta} e^{-\beta (s-t)}
            [r(s,x_s,a_s)-q_\psi(s,x_s,a_s)]\dif s-e^{-\beta\delta}V_\theta(t+\delta,x_{t+\delta})\Big), 
\end{equation*}
\begin{equation*}
    \psi\leftarrow\psi-\eta\partial_\psi q_\psi(t,x_t,a_t) \cdot\Big(V_\theta(t,x_t)-\int_t^{t+\delta} e^{-\beta (s-t)}
            [r(s,x_s,a_s)-q_\psi(s,x_s,a_s)]\dif s-e^{-\beta\delta}V_\theta(t+\delta,x_{t+\delta})\Big), 
\end{equation*}
where $\delta>0$ is the integral interval and the trajectory is sampled from collected data.
Note that the update formula above is also referred as semi-gradient TD method in RL  \citep{sutton1998reinforcement}.

\paragraph{Bellman constraints.}
To enforce \eqref{eq:hjb_condition}, we re-parameterize the advantage rate function as
\begin{equation}\label{eq:repara_q}
    q_\psi(t,x,a):=\bar{q}_\psi(t,x,a)-\bar{q}_\psi(t,x,\mu_\phi(t,x)),
\end{equation}
where $\bar{q}_\psi$ is a neural network and $\mu_\phi$ denotes the current deterministic policy \citep{tallec2019making}.

In practice, it is often challenging to design a neural network structure that directly enforces the terminal value constraint. To address this, we add a penalty term of the form: $\E (V_\theta(T,x_T)-g(x_T))^2$, where $x_T,g(x_T)$ are sampled from  collected trajectories.

\paragraph{Implementation with discretization.}
Let $h$ denote the discretization step size.
We denote by $\tilde{x}_t$ the concatenation of time and state $(t,x_t)$ for compactness.
The full procedure of \textbf{C}ontinuous \textbf{T}ime \textbf{D}eep \textbf{D}eterministic \textbf{P}olicy \textbf{G}radient (CT-DDPG) is summarized in \Cref{alg:ddpg_seq}. 

We employ several training techniques widely used in modern deep RL algorithms such as DDPG and SAC.
In particular, we employ a target value network $V_{\theta^{tgt}}$, defined as the exponentially moving average of the value network weights. 
This technique has been shown to improve training stability in deep RL.\footnote{Here we focus on a single target value network as our primary goal is to study the efficiency of deterministic policies in continuous-time RL. Extensions with multiple target networks \citep{fujimoto2018addressing} can be readily incorporated.} 
We further adopt a replay buffer to store transitions in order to improve sample efficiency. 
For exploration, we add independent Gaussian   noises to the deterministic policy $\mu_\phi$.

\paragraph{Multi-step TD.}
When training advantage-rate net and value net, we adopt multiple steps $L>1$ to compute the temporal difference (TD) error (see \eqref{eq:martingale_loss}), so that the objective is defined as
\begin{equation}\label{eq:multi_step_TD_CT}
    \min_{\theta,\psi}\E\Big(V_\theta(\tilde{x}_t)-\sum_{k=0}^{L-1}e^{-\beta kh}[r(\tilde{x}_{t+kh},a_{t+kh})-q_\psi(\tilde{x}_{t+kh},a_{t+kh})]-e^{-\beta Lh}V_{\theta^{tgt}}(\tilde{x}_{t+Lh})\Big)^2.
\end{equation}
This is different from the standard off-policy formulation in discrete time RL algorithms, which typically rely on a single transition step.
Notably, when $L=1$, our algorithm reduces to DAU \citep[Alg. 2]{tallec2019making} except that their policy learning rate vanishes as $h\to 0$.
We highlight that multi-step TD is essential for the empirical success of CT-DDPG (see \Cref{sec:exp}).

\begin{algorithm}[!htbp]
    \caption{\textbf{C}ontinuous \textbf{T}ime \textbf{D}eep \textbf{D}eterministic \textbf{P}olicy \textbf{G}radient}
    \label{alg:ddpg_seq}
    \begin{algorithmic}
        \STATE{{\bfseries Inputs:} Discretization step size $h$, horizon $K=T/h$, discount rate $\beta$, number of episodes $N$,  policy net $\mu_\phi$, advantage-rate net $\bar{q}_\psi$, value net $V_\theta$, update frequency $m$, trajectory length $L$, exploration noise $\sigma_{\text{explore}}$, soft update parameter $\tau$, learning rate $\eta$, batch size $B$, terminal value constraint weight $\alpha$}
        \STATE{
            {\bfseries Learning Procedures:}
            % \STATE{
            %     Initialize $\phi,\psi,\theta$ and target $\phi^{tgt}=\phi,\psi^{tgt}=\psi,\theta^{tgt}=\theta$ 
            % }
            \STATE{
                Initialize $\phi,\psi,\theta$, target $\theta^{tgt}=\theta$, and replay buffer $\mathcal{R}$
            }
            \FOR{$n = 1, \cdots, N$}
                \STATE{Observe the initial state $\tilde{x}_0$}
                \FOR{$k=1,\cdots,K$}
                    \STATE{
                        Perform $a_{kh}\sim \mathcal{N}(\mu_\phi(\tilde{x}_{kh}),\sigma_{\text{explore}}^2)$ and collect $r_{kh}, \tilde{x}_{(k+1)h}$
                    }
                    \STATE{
                        Store $(\tilde{x}_{kh},a_{kh},r_{kh}, \tilde{x}_{(k+1)h})$ in $\mathcal{R}$
                    }
                    \IF{$k \equiv 0 \textbf{ mod } m$}
                    \STATE{$\triangleright$\colorbox{Ocean}{\emph{train advantage rate function and value function}}}
                    
                    \STATE{
                        Sample a batch of trajectories $\{\tilde{x}_{k_ih:(k_i+L)h}^{(i)}, a_{k_ih:(k_i+L)h}^{(i)}, r_{k_ih:(k_i+L)h}^{(i)}\}_{i=1}^B$ from $\mathcal{R}$
                    }
                    \STATE{
                        Define $q_\psi(\tilde{x},a):=\bar{q}_\psi(\tilde{x},a)-\bar{q}_\psi(\tilde{x},\mu_\phi(\tilde{x}))$
                    }
                    \STATE{
                        Compute the martingale loss \begin{equation}\label{eq:martingale_loss}
                            \mathcal{L}^{M}=\frac{1}{B}\sum_{i=1}^B \Big(V_\theta(\tilde{x}_{k_ih}^{(i)})-\sum_{l=0}^{L-1}e^{-\beta lh}[r_{(k_i+l)h}^{(i)}-q_\psi(\tilde{x}_{(k_i+l)h}^{(i)},a_{(k_i+l)h}^{(i)})]h-e^{-\beta Lh}V_{\theta^{tgt}}(\tilde{x}_{(k_i+L)h}^{(i)})\Big)^2
                        \end{equation}
                        % \begin{equation}
                        %     \mathcal{L}^{M}=\E \Big(V_\theta(\tilde{x}_{t})-\sum_{l=0}^{L-1}e^{-\beta lh}[r_{t+lh}-q_\psi(\tilde{x}_{t+lh},a_{t+lh})]h-e^{-\beta Lh}V_{\theta^{tgt}}(\tilde{x}_{t+Lh})\Big)^2
                        % \end{equation}
                        % \begin{equation}
                        %     Q(t,x,a):=V(t,x)+h\cdot q(t,x,a)
                        % \end{equation}
                    }
                    
                    \STATE{
                        Sample a batch of terminal states $\{\tilde{x}_{Kh}^{(i)}, r_{Kh}^{(i)}\}_{i=1}^B$ from $\mathcal{R}$
                    }
                    \STATE{
                        Compute the terminal value constraint $\mathcal{L}^C=\frac{1}{B}\sum_{i=1}^B(V_\theta(\tilde{x}_{Kh}^{(i)})-r_{Kh}^{(i)})^2$
                    }
                    \STATE{
                        Update the critic: $\psi\leftarrow\psi-\eta\partial_\psi (\mathcal{L}^M+\alpha\mathcal{L}^C)$,
                        $\theta\leftarrow\theta-\eta\partial_\theta (\mathcal{L}^M+\alpha\mathcal{L}^C)$
                    }
                    \STATE{$\triangleright$\colorbox{Ocean}{\emph{train policy}}}
                    
                    \STATE{
                        Sample a batch of states $\{\tilde{x}_{k_ih}^{(i)}\}_{i=1}^B$ from $\mathcal{R}$
                    }
                    \STATE{
                        Compute the policy loss $\mathcal{L}=\frac{1}{B}\sum_{i=1}^B\bar{q}_\psi(\tilde{x}_{k_ih}^{(i)},\mu_\phi(\tilde{x}_{k_ih}^{(i)}))h$
                    }
                    \STATE{
                        Update the actor: $\phi\leftarrow \phi+\eta \partial_\phi\mathcal{L}$
                    }
                    \STATE{
                        Update the target: 
                        $\theta^{tgt}\leftarrow \tau \theta+(1-\tau)\theta^{tgt}$
                    }
                    \ENDIF
                \ENDFOR
            \ENDFOR
        }
    \end{algorithmic}
\end{algorithm}

Indeed, multi-step TD objective has also been studied in discrete-time RL \citep{hessel2018rainbow}. Given a
stochastic policy $\pi(\cdot|\cdot)$ (possibly deterministic) and a horizon $L > 1$, the $TD(L)$ objective is defined as
\begin{equation}\label{eq:multi_step_TD}
    \min_\psi\E\Big(Q_\psi(\tilde{x}_t,a_t)-\sum_{k=0}^{L-1}e^{-\beta kh}r(\tilde{x}_{t+kh},a_{t+kh})-e^{-\beta Lh}Q_{\psi^{tgt}}(\tilde{x}_{t+Lh},a_{t+Lh})\Big)^2.
\end{equation}
Here $a_{t+kh}\sim\pi(\cdot|\tilde{x}_{t+kh})$ for any $k\geq 1$. However, in the off-policy setting, the trajectory $\{(\tilde{x}_{t+kh},\tilde{a}_{t+kh})\}_{1\leq k\leq L-1}$ is typically drawn from the replay buffer, i.e., the historical data set, which leads to a distribution shift issue. By contrast, the objective \eqref{eq:multi_step_TD_CT} used in CT-DDPG avoids this issue by introducing an additional advantage-rate function, thereby enabling the use of historical trajectories from the replay buffer to improve sample efficiency.

In the following section, we will theoretically demonstrate that one-step TD inevitably leads to gradient variance blow-up in the limit of vanishing discretization step, thereby slowing convergence.
This is, to the best of our knowledge,  the first theoretical analysis regarding the stability of multi-step TD for continuous-time RL problems. 

\subsection{Analysis of Multi-Step TD in Continuous RL}\label{subsec:var_blow_up}

When training the value function $V_\theta$ and the advantage function $A_\psi$
for   a given policy (stochastic or deterministic),  \textit{Temporal Difference} algorithms \citep{haarnoja2018soft,tallec2019making,jia2023q} typically use a one-step semi-gradient:
\begin{equation}\label{eq:semi_grad}
    \begin{aligned}
        &G_{\theta,h}:= \frac{1}{h}\E\left[\partial_\theta V_\theta(\tilde{x}_t)\left(V_\theta(\tilde{x}_t)-(r_t-A_\psi(\tilde{x}_t,a_t))\cdot h-e^{-\beta h}V_\theta(\tilde{x}_{t+h})\right)\right], \\
        &G_{\psi,h}:=\frac{1}{h}\E\left[\partial_\psi A_\psi(\tilde{x}_t,a_t)\left(V_\theta(\tilde{x}_t)-(r_t-A_\psi(\tilde{x}_t,a_t))\cdot h-e^{-\beta h}V_\theta(\tilde{x}_{t+h})\right)\right],
    \end{aligned}
\end{equation}
where $t\sim TruncExp(\beta;T)$ and $x_t\sim X_t^{\pi'}, a_t\sim\pi'(\cdot|t,x_t)$ with an exploration policy $\pi'$.
In practice, however, one has to use stochastic gradient:
\begin{equation}
    \begin{aligned}
        &g_{\theta,h}:=\frac{1}{h}\left[\partial_\theta V_\theta(\tilde{x}_t)\left(V_\theta(\tilde{x}_t)-(r_t-A_\psi(\tilde{x}_t,a_t))\cdot h-e^{-\beta h}V_\theta(\tilde{x}_{t+h})\right)\right]. \\
        &g_{\psi,h}:=\frac{1}{h}\left[\partial_\psi A_\psi(\tilde{x}_t,a_t)\left(V_\theta(\tilde{x}_t)-(r_t-A_\psi(\tilde{x}_t,a_t))\cdot h-e^{-\beta h}V_\theta(\tilde{x}_{t+h})\right)\right].
    \end{aligned}
\end{equation}
This one-step TD risk variance blow-up as the next proposition shows.
\begin{thm}
    
\label{prop:var_blow_up}
    % Assume $\underline{C}\cdot I\preceq\sigma\sigma^\top\preceq \overline{C}\cdot I$ for some $0<\underline{C}\leq\overline{C}$, 
    {Assume that $\partial_\theta V_\theta, \|\partial_xV_\theta\|_{\sigma\sigma^\top}$ are not identically zero}. 
    Then the variance of stochastic gradient estimator blows up in the sense that:
    \begin{equation}
        \lim_{h\to 0}\E[g_{\theta,h}]=\lim_{h\to 0}G_{\theta,h}=\Theta(1), \quad 
        \lim_{h\to 0}\E[g_{\psi,h}]=\lim_{h\to 0}G_{\psi,h}=\Theta(1),
    \end{equation}
    \begin{equation}
        \lim_{h\to 0}h\cdot\Var(g_{\theta,h})=\Theta(1), \quad  \lim_{h\to 0}h\cdot\Var(g_{\psi,h})=\Theta(1).
    \end{equation}
    % Then the variance of stochastic gradient estimator blows up in the sense that:
    % \begin{equation}
    %     \lim_{h\to 0}\E[g_{\theta,h}]=\lim_{h\to 0}G_{\theta,h}=\Theta(1),\quad
    %     \lim_{h\to 0}h\cdot\Var(g_{\theta,h})=\Theta(1).
    % \end{equation}
\end{thm}

In contrast, \Cref{alg:ddpg_seq} utilizes $L$-step TD loss with (stochastic) semi-gradient (for simplicity of the theoretical analysis, we consider hard update of target, i.e., $\tau=1$):
\begin{equation}\label{eq:semi_grad_seq}
    \hspace{-2mm}G_{\theta,h,L}=\E\big[\partial_\theta V_\theta(\tilde{x}_t)\big(V_\theta(\tilde{x}_t)-\sum_{l=0}^{L-1}e^{-\beta lh}[r_{t+lh}-q_\psi(\tilde{x}_{t+lh},a_{t+lh})]h-e^{-\beta Lh}V_\theta(\tilde{x}_{t+Lh})\big)\big],
\end{equation}
\begin{equation}
    g_{\theta,h,L}=\partial_\theta V_\theta(\tilde{x}_t)\big(V_\theta(\tilde{x}_t)-\sum_{l=0}^{L-1}e^{-\beta lh}[r_{t+lh}-q_\psi(\tilde{x}_{t+lh},a_{t+lh})]h-e^{-\beta Lh}V_\theta(\tilde{x}_{t+Lh})\big).
\end{equation}

\begin{thm}\label{prop:var_normal}
Under the same assumptions in \Cref{prop:var_blow_up}, if $Lh\equiv\delta>0$, then the expected gradient does not vanish in the sense that 
\begin{equation}
    \lim_{h\to 0}\E[g_{\theta,h,\frac{\delta}{h}}]=\lim_{h\to 0}G_{\theta,h,\frac{\delta}{h}}=\Theta(1).
\end{equation}
In addition, the variance of stochastic gradient does not blow up: $\overline{\lim_{h\to 0}}\Var(g_{\theta,h,\frac{\delta}{h}})=\mathcal{O}(1)$.
\end{thm}

\begin{rmk}[Effect of $1/h$ scaling]
    Note that  \eqref{eq:semi_grad_seq}  drops the $1/h$ factor in contrast to \eqref{eq:semi_grad}. This modification is crucial for preventing the variance from blowing up. If we were to remove the $1/h$ factor in \eqref{eq:semi_grad}, then according to \Cref{prop:var_blow_up} the expected gradient $G_{\theta,h}$ would vanish as $h\to 0$. This theoretical inconsistency reveals a fundamental drawback of one-step TD methods in the continuous-time RL framework, which is also verified in our experiments.
\end{rmk}

\begin{rmk}[Previous analysis of one-step TD]
    \citep{jia2022policy_eval} discussed one-step TD objective
    \begin{equation}\label{eq:one_step_TD_square}
        \min_\theta \frac{1}{h^2}\E_{\tilde{x}}\left(V_\theta(\tilde{x}_t)-r_t\cdot h-e^{-\beta h}V_\theta(\tilde{x}_{t+h})\right)^2,
    \end{equation}
    showing that its minimizer does not converge to the true value function as $h\to 0$.
    However, practical one-step TD methods do not directly optimize \eqref{eq:one_step_TD_square}, but rather employ the semi-gradient update \eqref{eq:semi_grad}.
    Consequently, the analysis in \citep{jia2022policy_eval} does not fully explain the failure of discrete-time algorithms under small discretization steps. 
    In contrast, our analysis is consistent with the actual update rule and thus offers theoretical insights of continuous-time algorithmic designs.
\end{rmk}

\section{Numerical Experiments}\label{sec:exp}

The goal of our numerical experiments is to evaluate the efficiency of the proposed CT-DDPG algorithm in general continuous time-and-state RL framework, especially in comparison with the stochastic policies,  in terms of convergence speed, training stability and robustness with respect to varing discretization steps and noise levels.

\subsection{Linear Quadratic Regulator}
Consider the following one-dimensional linear dynamic
\begin{equation}\label{eq:lq_dynamic}
    \dif X_t=(AX_t+B\alpha_t)\dif t + (CX_t+D\alpha_t)\dif W_t, \quad X_0\sim\mathrm{Unif}([a,b]),
\end{equation}
with a quadratic reward function
\begin{equation}\label{eq:lq_reward}
    \ J:=-\E\left[\int_0^T (QX_t^2+R\alpha_t^2+2SX_t\alpha_t)\dif t+G(X_T-w)^2\right],
\end{equation}
for some constants   $A, B, C, D, Q, R, S, G$   given below. The optimal value function can be explicitly computed by solving a Riccati equation and the corresponding optimal policy is linear in state.

We evaluate CT-DDPG against the stochastic policy-based learning methods in \cite{jia2023q} under two settings: (i) \emph{model-aware}, where the LQ structure is known and the policy and value function are parameterized according to the corresponding optimal form; and (ii) \emph{model-agnostic}, where the LQ structure is unknown and all functions are parameterized using neural networks. In both settings, the specific parameters in \eqref{eq:lq_dynamic}-\eqref{eq:lq_reward} are neither known nor accessible.

For model-aware LQ, we set $A=C=0,B=0.1,D=0.2$ and $Q=R=S=0,G=1,w=1$, following \citep{jia2023q}. In this case, the   q-learning algorithm in \citep{jia2023q} can be implemented since the advantage rate function (i.e., q-function in \citep{jia2023q}) is quadratic and the corresponding policy is Gaussian.
The results are shown in \Cref{fig:mv}.
For relatively large stepsize ($h=10^{-3}$), q-learning exhibits faster converges as it explicitly exploits the LQ structure, whereas CT-DDPG must perform policy improvement through policy gradient estimation.
However, when the stepsize becomes extremely small ($h=10^{-5}$), the performance of q-learning deteriorates since it relies on one-step TD to learn the advantage rate function and thus suffers from variance explosion issues described in \Cref{prop:var_blow_up}. In contrast, CT-DDPG employs multi-step TD with $L=100$ and is therefore more robust to the discretization issue. 
\begin{figure}[ht]
    \centering
    \includegraphics[width=0.7\linewidth]{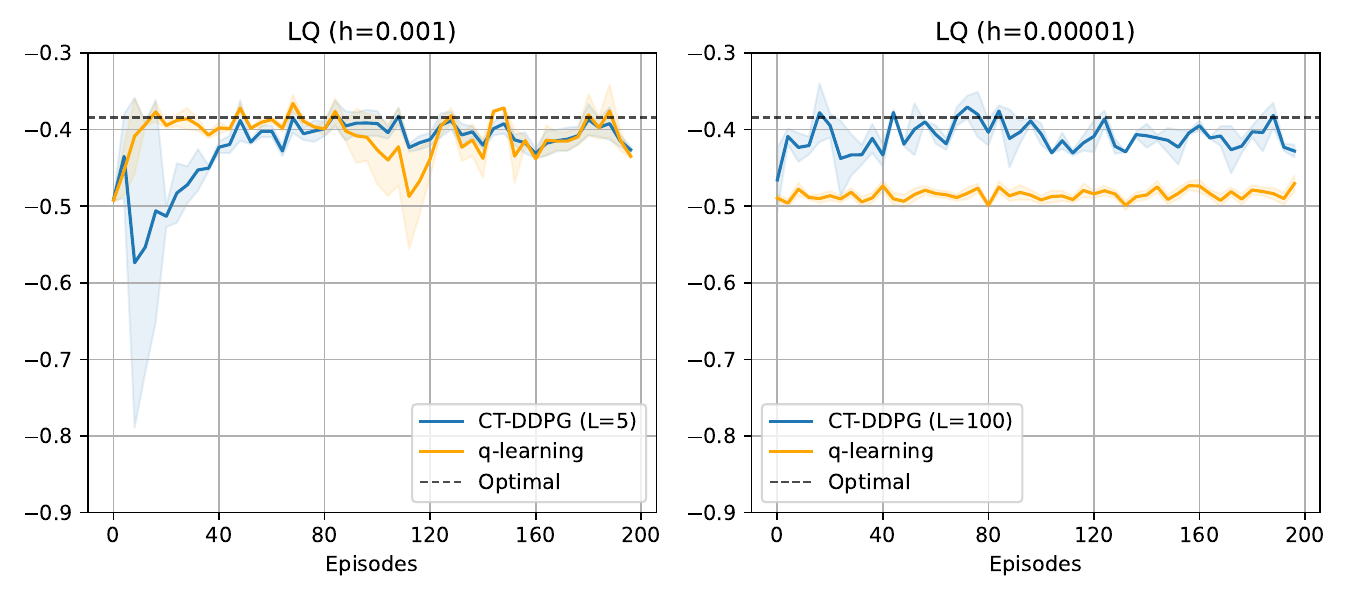}
    \caption{Model-aware LQ with linear policy.}
    \label{fig:mv}
\end{figure}

For model-agnostic LQ, we set $A=1,B=C=0.1,D=0.2$ and $Q=R=S=1,G=0$ and implement policy, advantage rate and value function as three layer fully connected ReLU MLPs with $32$ hidden units. In this case, q-learning cannot be directly implemented and one has to apply stochastic policy gradient (CT-SPG, see Section 5 in \citep{jia2023q}). The results are shown in \Cref{fig:lq}.
CT-DDPG outperforms the CT-SPG counterpart regardless of the choices of the step size $h$. In all cases, stochastic policies fail to converge to the optimal policy since the Bellman consistency \eqref{eq:consistency_stochastic} cannot be strictly imposed.  

\begin{figure}[ht]
    \centering
    \includegraphics[width=0.7\linewidth]{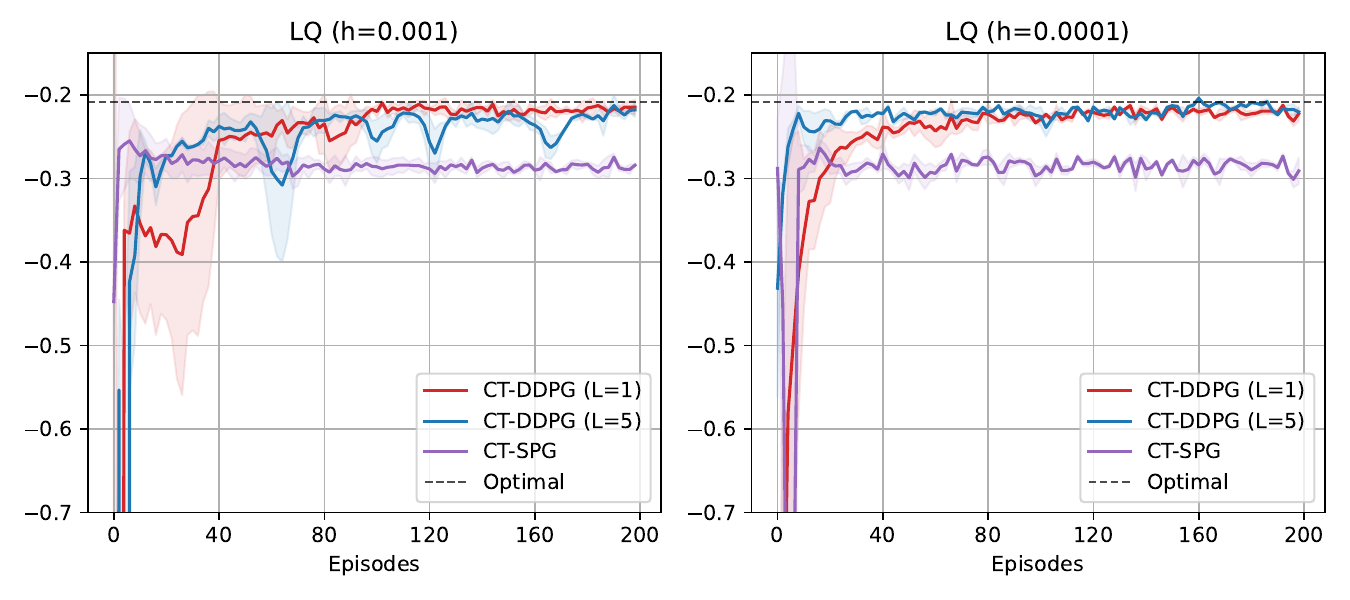}
    \caption{Model-agnostic LQ with neural network parameterized policy.}
    \label{fig:lq}
\end{figure}

\subsection{Robotic Control}

\paragraph{Environments.}
We evaluate on a suite of challenging continuous-control benchmarks from Gymnasium \citep{towers2024gymnasium}: \textit{Pendulum-v1}, \textit{HalfCheetah-v5}, \textit{Hopper-v5}, and \textit{Walker2d-v5}, sweeping the discretization step and dynamic noise levels. To model stochastic dynamics, at each simulator step we sample the i.i.d. Gaussian noises $\xi\sim \mathcal{N}(0,\sigma^2 I)$ and write it to MuJoCo’s \texttt{qfrc\_applied} buffer \citep{todorov2012mujoco}, thereby perturbing the equations of motion. We set \texttt{terminate\_when\_unhealthy=False} for MuJoCo environments. To accelerate training, we run 8 environments in parallel, i.e., collecting 8 trajectories per episode. 

\paragraph{Model architecture.}
Across all experiments, the policy, Q-network, and value network are implemented as three-layer fully connected MLPs with ReLU activations. The hidden dimension is set to 400, except for \textit{Pendulum}, where we use 64. To incorporate time information, we augment the environment observations with a sinusoidal embedding, yielding $\tilde{x}_t=(x_t,\cos(\tfrac{2\pi t}{T}),\sin(\tfrac{2\pi t}{T}))$, where $T$ denotes the maximum horizon. For stochastic policies, we employ Gaussian policies with mean and variance parameterized by neural networks, and fix the entropy coefficient to $\gamma=0.1$.

\paragraph{Training hyperparameters.}
We use the Adam optimizer with a learning rate of $3\times 10^{-4}$ for all networks ($3\times 10^{-3}$ for \textit{Pendulum}), and a batch size of $B=256$. The update frequency is $m=1$ in the original environment and $m=5$ for smaller step sizes $h$. The discount rate is set to $\beta=0.8$, applied in the form $e^{-\beta h}$. The soft target update parameter is $\tau=0.005$. The weight for the terminal value constraint is $\alpha=0.002$.  
For CT-DDPG, the trajectory length $L$ is sampled uniformly from $[2,10]$, and we use exploration noise with standard deviation $\sigma_{\text{explore}}=0.1$.  
For q-learning, for each state $\tilde{x}$ in the minibatch, we sample $n=20$ actions from $\pi(\cdot \mid \tilde{x})$ and compute the penalty term $\big(\frac{1}{n}\sum_{i=1}^n \big[q_\psi(\tilde{x},a_i)-\gamma\log \pi(a_i\mid \tilde{x})\big]\big)^2$.

\paragraph{Baselines.}
We compare against discrete-time algorithms DDPG \citep{lillicrap2015continuous}, SAC \citep{haarnoja2018soft_application}, DDPG equipped with multi-step TD (DDPG-TD(L), \citep{hessel2018rainbow}) 
as well as a continuous-time algorithm with stochastic Gaussian policy: CT-SPG \citep{jia2023q}. 
In particular, for CT-SPG, we adopt two different settings when learning q-function: the original one-step TD taregt ($L=1$) in \citep{jia2023q}, and a multi-step TD extension with $L>1$ as in \Cref{alg:ddpg_seq}.
This provides  fair  comparisons between deterministic and stochastic policies in continuous-time RL.
We also test DAU \citep{tallec2019making}, i.e., CT-DDPG with $L=1$, to see the effects of multi-step TD.
For each algorithm, we report results averaged over at least three independent runs with different random seeds.

\begin{figure}[ht]
    \centering
    \includegraphics[width=0.95\linewidth]{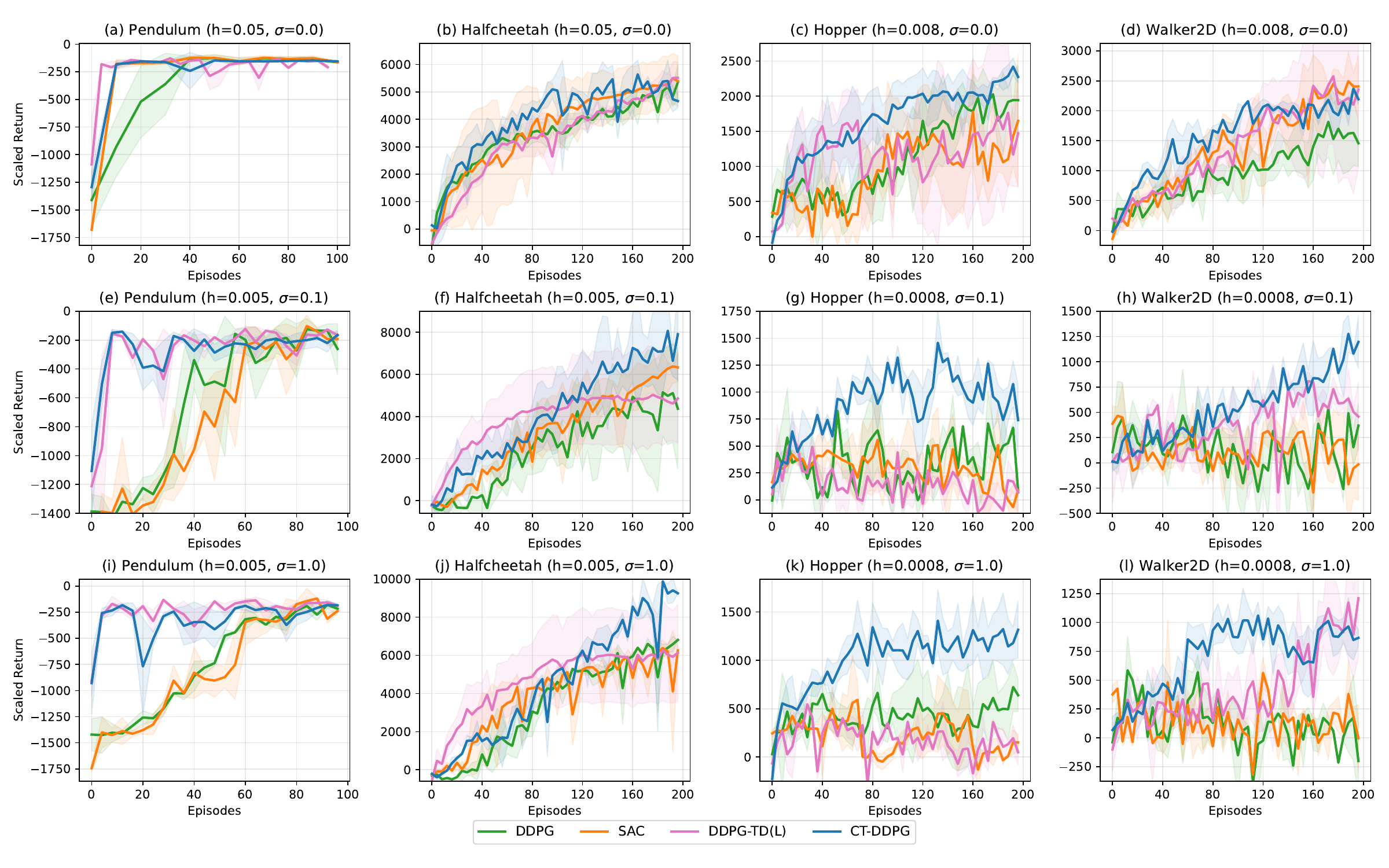}
    \caption{Comparison between CT-DDPG with discrete-time RL algorithms.}
    \label{fig:discrete_comparison}
\end{figure}

\paragraph{Results.} \Cref{fig:continuous_comparison,fig:discrete_comparison} show the average return against training episodes, where the shaded area stands for standard deviation across different runs. 
We observe that for most environments, our CT-DDPG has the best performance among all baselines and the gap becomes larger as discretization step decreases and (or) noise level increases.
Specifically, 
\begin{itemize}[leftmargin=1em]
    \item As demonstrated in \Cref{fig:discrete_comparison}, although discrete-time algorithms, DDPG and SAC, perform reasonably well under the standard Gymnasium settings (top row), they degrade substantially when $h$ decreases and $\sigma$ increases (middle \& bottom rows). This stems from the fact that one-step TD updates provide only myopic information under small $h$ and noisy dynamics, preventing the Q-function from capturing the long-term structure. Incorporating DDPG with multi-step TD can partially alleviate this issue; however, it still exhibits slow convergence due to the persistent bias in the TD objective induced by off-policy sampling.
    \item For continuous-time RL with stochastic policy shown in \Cref{fig:continuous_comparison}, q-learning exhibits slow convergence and training instability, due to the difficulty of enforcing Bellman equation constraints \eqref{eq:consistency_stochastic}. Although q-learning using multi-step TD can to some extent improve upon original q-learning ($L=1$), it still remains unstable across diverse environment settings and underperforms compared to CT-DDPG. This highlights the fundamental limitations of stochastic policy in continuous-time RL.
    \item To further investigate the effects of multi-step TD, we also test DAU (i.e., CT-DDPG with $L=1$) in \Cref{fig:continuous_comparison}. It turns out that in small $h$ and large $\sigma$ regime, DAU converges more slowly. In \Cref{fig:nsr}, we examine the variance to square norm ratio (NSR) of stochastic gradients in the training process and report the moving average. As $h\to 0$, NSR of DAU becomes evidently larger than that of CT-DDPG, consistent with our theories in \Cref{subsec:var_blow_up}. A large NSR leads to the instability when training q-function and consequently impedes the convergence. 
\end{itemize}

\begin{figure}[ht]
    \centering
    \includegraphics[width=0.95\linewidth]{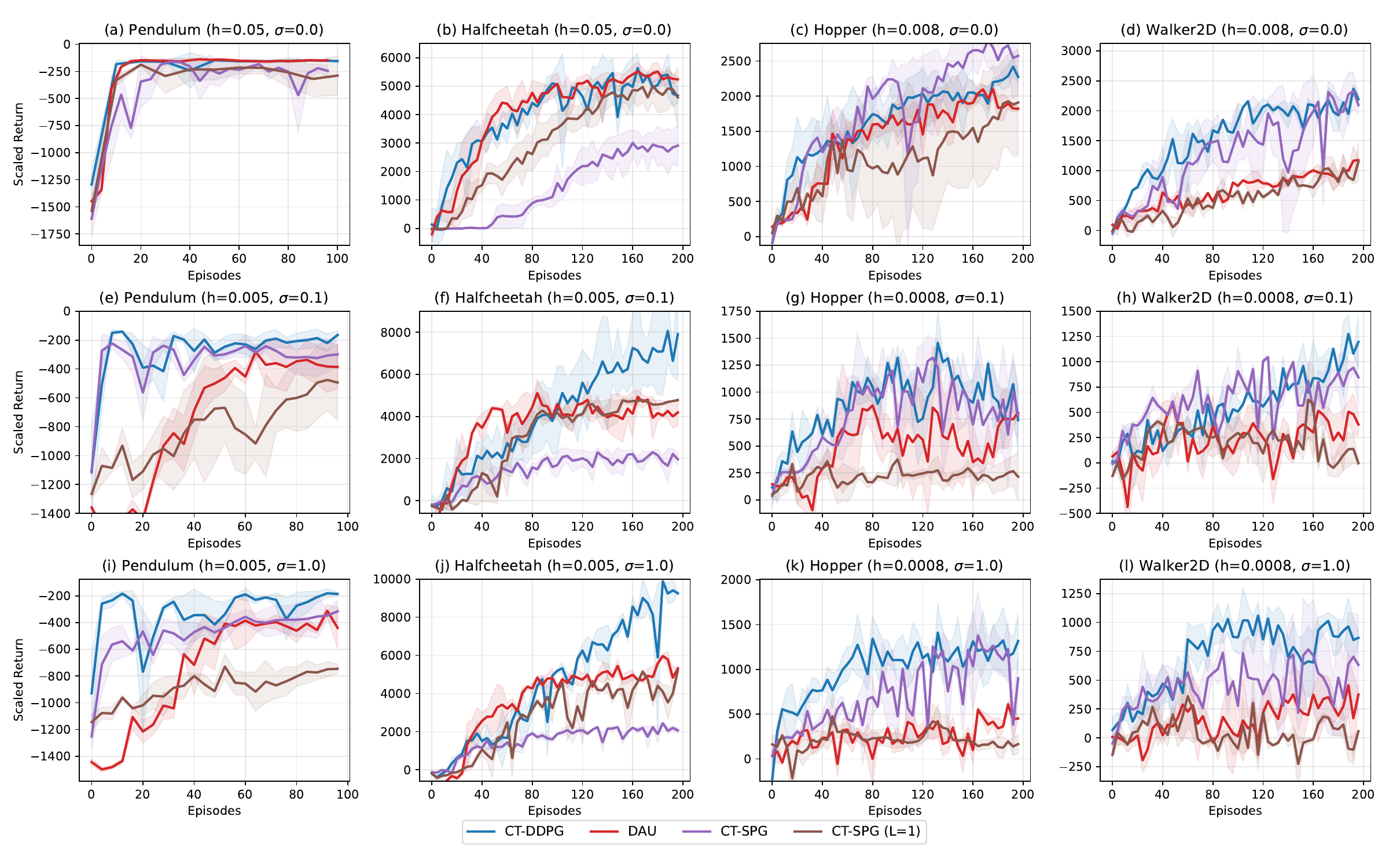}
    \caption{Comparison between continuous-time RL algorithms.}
    \label{fig:continuous_comparison}
\end{figure}
\begin{figure}[ht]
    \centering
    \includegraphics[width=0.95\linewidth]{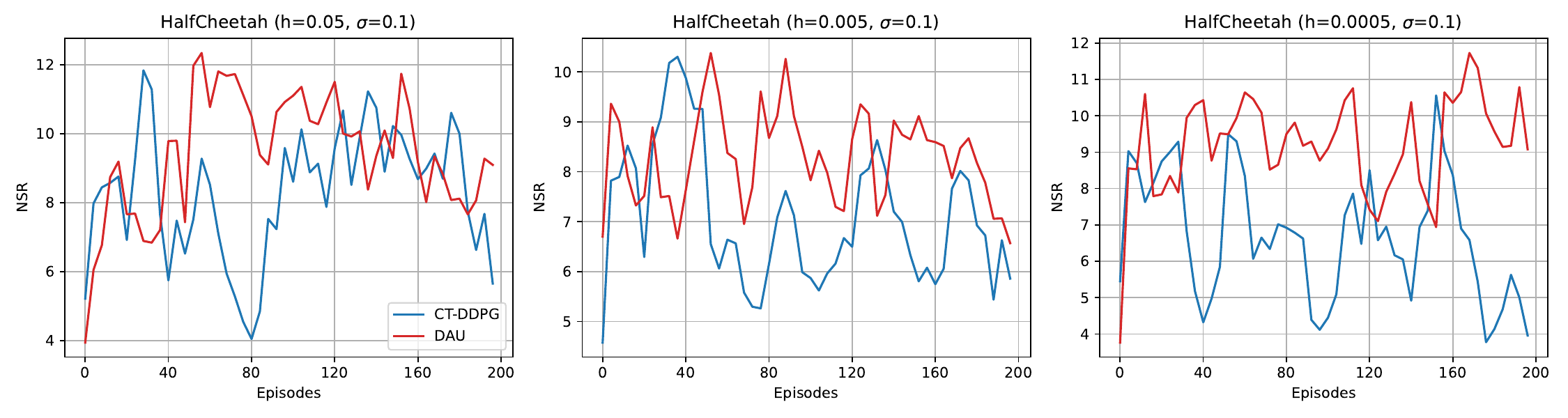}
    \caption{Noise-to-Signal Ratio of stochastic gradient when training value-net $V_\theta$.}
    \label{fig:nsr}
\end{figure}

In summary, CT-DDPG exhibits superior performance in terms of convergence speed and stability across most environment settings, verifying the efficiency and robustness of our method.

\section{Proofs of Main Results}

\subsection{Proofs of \Cref{prop:pg}}

The following performance difference lemma   characterizes   the difference of  value functions  
  with  different policies,
  which will be used in proving the policy gradient formula. 
\begin{lemma}
\label{lemma:performance_difference}
    Suppose \Cref{assum:wp} holds.
  Let 
     $\phi\in \sR^k$ 
     and assume  
     $
    V^\phi   \in C^{1,2}([0, T]  \times \sR^n)$.
  For all $(t,x)\in [0,T]\times \sR^n$
 and $\phi'\in \sR^k$,
 \begin{align}
 \begin{split}
&  V^{\phi'}(t,x )
  -V^{\phi }(t,x)
  \\
& =  
     \mathbb E \bigg[ \int_t^T
   e^{-\beta(s-t)} 
    \bigg(  H[V^\phi]    (s, X^{  \phi' }_s,\mu_{\phi'}(s,X^{  \phi' }_s) )
   -H[V^\phi]  (s, X^{  \phi' }_s, \mu_\phi(s,X^{  \phi' }_s) ) 
   \bigg)
   \dif s \,\bigg\vert\, X^{  \phi' }_t =x\bigg].
 \end{split}
 \end{align}
\end{lemma}

\begin{proof}[Proof of \Cref{prop:pg}]
Recall that 
 $\partial_\phi V^\phi(t,x)  =(
 \partial_{\phi_1} V^\phi(t,x),
 \ldots, 
 \partial_{\phi_k} V^\phi(t,x) 
 )^\top $.
 Hence it suffices to prove 
for all   $   \phi'\in \sR^k$,
 \begin{align*}
 \begin{split}
&  \frac{\dif }{\dif \epsilon }V^{\phi+\epsilon \phi'}(t,x)
  \bigg\vert_{\epsilon=0}
   =  
     \mathbb E \left[ \int_t^T
   e^{-\beta(s-t)} 
    \partial_a A^\phi(s, X^{  \phi }_s,\mu_\phi(s,X^\phi_s))^\top  \partial_\phi \mu_\phi(s,X^\phi_s)
   \dif s\,\bigg\vert\, X^\phi_t=x \right]  \phi'.
 \end{split}
 \end{align*}

 To this end,
     for all $\epsilon \in [-1,1]$,
    let $X^{  \epsilon } $
    be the solution to the following dynamics:
    \begin{equation}
    \label{eq:sde_epsilon}
    \dif X_s=b(s,X_s,\mu_{\phi+\epsilon \phi'}(s, X_s))\dif s+\sigma(s,X_s,\mu_{\phi+\epsilon \phi'}(t,X_t))\dif W_s, 
    \quad X_t=x.
\end{equation}
For all $\epsilon\in[-1,1]$,
by \Cref{lemma:performance_difference}
and the fundamental theorem of calculus, 
   \begin{align}
   \label{eq:derivative_proof1}
 \begin{split}
& \frac{V^{\phi+\epsilon \phi'}(t,x)
  -V^{\phi }(t,x)  }{\epsilon}
  =  
     \mathbb E \bigg[ \int_t^T
   e^{-\beta(s-t)} 
   \left(  \int_0^1
    \mathcal G(s,X^{  \epsilon }_s,\phi+r\epsilon \phi')  \dif r  
   \right)
   \dif s \bigg]   \phi',
 \end{split}
 \end{align}
 where for all $\tilde\phi\in \sR^k$, 
 $$
 \mathcal G(t,x,\tilde\phi)
 \coloneqq   \partial_a H[ V^\phi ](t, x,\mu_{\tilde\phi}(s,x))^\top \partial_\phi \mu_{{\tilde\phi}}(t,x). 
 $$
  To show the  limit of \eqref{eq:derivative_proof1} as $\epsilon \to 0$, 
observe that 
    by \Cref{assum:differentiability} and standard stability analysis of \eqref{eq:sde_epsilon} (see e.g., \cite[Theorem 3.2.4]{zhang2017backward}),
   for all $\epsilon \in [-1,1]$,
\begin{align*}
\begin{split}
     \mathbb{E} \left[ \sup_{t \leq s \leq T} |X_s^{\epsilon}-X_s^{0} |^2 \right]
    &
    \leq C 
     \mathbb{E} \left[
     \left(\int_0^T |b(s, X_s^{0},  \mu_{\phi+\epsilon \phi'}(s, X_s^{0}))
     -b(s, X_s^{0},  \mu_{\phi}(s, X_s^{0})) | \dif s\right)^2
     \right]
     \\
     &\quad + C 
     \mathbb{E} \left[
      \int_0^T |\sigma(s, X_s^{0},  \mu_{\phi+\epsilon \phi'}(s, X_s^{0}))
     -\sigma(s, X_s^{0},  \mu_{\phi}(s, X_s^{0}))|^2 \dif s 
     \right],
\end{split}
\end{align*}
which along with the growth condition in \Cref{assum:wp}
and the   regularity  of $b$, $\sigma$ and $\phi$ in 
\Cref{assum:differentiability}, and the dominated convergence theorem shows that 
\begin{equation}
\label{eq:state_convergence}
 \lim_{\epsilon\to 0}\mathbb{E} \left[ \sup_{t \leq s \leq T} |X_s^{\epsilon}-X_s^{0} |^2 \right]=0. 
  \end{equation}
Moreover,
there exists $C\ge 0$ such that for all $\epsilon \in [-1,1]$, and
$A\in \mathcal F\otimes \mathcal B([0,T])\otimes \mathcal B([0,1])$,
\begin{align*}
  & \mathbb E \bigg[ \int_t^T   \int_0^1  
   \mathbf{1}_A e^{-\beta(s-t)} 
  \left|  \mathcal G(s,X^{  \epsilon }_s,\phi+r\epsilon \phi')  \right|  \dif r  
   \dif s \bigg]  
   \\
   &
   \le \mathbb E \bigg[ \int_t^T   \int_0^1 
 \mathbf{1}_A \dif r  
   \dif s \bigg]^{\frac{1}{2}  }
   \mathbb E \bigg[ \int_t^T   \int_0^1  
    e^{-2\beta(s-t)} 
  \left|   \mathcal G(s,X^{  \epsilon }_s,\phi+r\epsilon \phi')  \right|^2  \dif r  
   \dif s \bigg]^{\frac{1}{2}  }
   \\
   &\le \mathbb E \bigg[ \int_t^T   \int_0^1 
 \mathbf{1}_A \dif r     \dif s \bigg]^{\frac{1}{2}  } 
   C \left(1+
   \mathbb{E} \left[ \sup_{ t \leq s \leq T} |X_s^{\epsilon} |^4 \right]
   \right)^{\frac{1}{2}  },
\end{align*}
where the last inequality used 
 the growth conditions on    the derivatives of the coefficients $b,\sigma,r$ and $\mu$, and of the value function   $V^\phi$.
Using the moment condition 
$\sup_{\epsilon\in [-1,1]}\mathbb{E} \left[ \sup_{ t \leq s \leq T} |X_s^{\epsilon} |^4 \right]<\infty$,
the   random variables  
 $\{(\omega, s,r)\mapsto e^{-\beta(s-t)} 
   \mathcal G(s,X^{  \epsilon }_s,\phi+r\epsilon \phi') \mid \epsilon \in [-1,1]\}$  are uniformly integrable.  
Hence using Vitali's convergence theorem and passing  $\epsilon \to 0$ 
in \eqref{eq:derivative_proof1} yield the desired identity. 
\end{proof}

\subsection{Proof of \Cref{thm:martingale_char}}

\begin{proof}[Proof of \Cref{thm:martingale_char}]

For all $(t,x)\in[0,T]\times\R^n$ and $a\in \mathcal O_{\mu_\phi(t,x)}$,
   applying It\^o's formula to $u\mapsto e^{-\beta (u-t)}\hat{V}(s,X_u^{t,x, a })$ yields
   for all $0\leq t<s\leq T$,
    \begin{equation}
        \begin{aligned}
            e^{-\beta (s-t)}\hat{V}(s,X_s^{t,x,  a })-  \hat{V}(t,X_t^{t,x,  a })& = \int_t^s e^{-\beta  (u-t)}\mathcal L[\hat V] 
            (u, X_u^{t,x, a },\alpha_u )
            \dif  u 
           \\
           &
            + \int_t^s e^{-\beta (u-t)}\partial_x\hat{V}(u, X_u^{t,x, a }  )^\top  \sigma(u, X_u^{t,x, a },\alpha_u )\dif  W_u.
        \end{aligned}
   \end{equation}
This along with the martingale condition \eqref{eq:martingale}
implies 
 \begin{equation*}
      \left( \int_t^s  
       e^{-\beta  (u-t)}
       \left(\mathcal L[\hat V]+r -\hat q\right)
            (u, X_u^{t,x, a },\alpha_u ) \dif u
            \right)_{s\in [t,T]}
   \end{equation*}
   is a  martingale, which has continuous paths and finite variation.  Hence almost surely
   \begin{equation}
    \label{eq:H-q_constant_mv}
   \int_t^s  
       e^{-\beta  (u-t)}
       \left(\mathcal L[\hat V]+r -\hat q\right)
            (u, X_u^{t,x, a },  \alpha_u ) \dif u = 0, \quad \forall s\in [t,T].
   \end{equation}    

We claim $(\mathcal L[\hat V]+r-\hat q )(t,x,a)= 0$
for all $(t,x)\in[0,T]\times\R^n$ and $a\in \mathcal O_{\mu_\phi(t,x )}$.
To see it, define 
$f(t,x,a)\coloneqq  (L[\hat V]+r -\hat q)
            (t, x, a ) $
       for all $(t,x, a)\in[0,T]\times\R^n\times \mathcal{A} $.
By assumptions, $f\in C([0,T]\times \sR^n\times \mathcal A ) $.
 Suppose there exists 
 $(\bar t,\bar x)\in [0,T]\times \sR^n$
 and $\bar a\in \mathcal O_{\mu_\phi(t,x )}$
 such that $f(\bar t,\bar x, \bar a)\not = 0$.  
 Due to the continuity of $f$, 
 we can assume without loss of generality that 
   $f(\bar t,\bar x, \bar a)>0$
   and $\bar t\in [0,T)$. The continuity of $f$ implies that there exist constants $\epsilon, \delta>0$ such that $f(t,x,  a)\ge \epsilon>0$ for all $(t,x, a)\in [0,T]\times \sR^n\times \mathcal{A}$ with $\max\{|t-\bar t|, |x-\bar x|,   |a-\bar a|\}\le \delta$.
   Now  consider the process  
 $X^{\bar t,\bar x, \bar a }$ defined by \eqref{eq:martingale_state_mv}, and 
define the stopping time 
   \begin{align*}
       \tau &\coloneqq \inf\left\{t\in [  \bar t,T]\mid  \max\{|t-\bar t|, |X^{\bar t,\bar x, \bar a }_t-\bar x|,
       |\alpha_t-\bar a|
       \}>\delta\right\}.
   \end{align*}
Note that 
$\tau >\bar t$ almost surely,
due to the sample path continuity of 
  $t\mapsto X^{\bar t,\bar x, \bar a }_t$
  and the  condition $\lim_{s\searrow t}   \alpha_s=\bar a$. 
  This along with \eqref{eq:H-q_constant_mv}  implies that 
there exists a measure zero set $\mathcal N$ such that 
for all $\omega\in \Omega\setminus \mathcal N$, 
   % $t\mapsto \left( X^{\bar t,\bar x,  \bar a }_t(\omega),  \alpha_t(\omega),   \sP_{X^{  \bar t,\bar \xi, \bar a }_t}\right)$ is continuous,
   $\tau(\omega)> \bar t $, 
   and 
$$
   \int_{\bar t}^{\tau(\omega)}  
       e^{-\beta  (u-\bar t)}
       f        \left(u, X_u^{\bar t,\bar x, \bar a }(\omega),  \alpha_u (\omega) \right) \dif u = 0.
   $$   
However, by the definition of $\tau$, for all $t\in (\bar t,\tau(\omega))$,
$ \max\{|t-\bar t|, |X^{\bar t,\bar x, \bar a }_t-\bar x|,
       |\alpha_t-\bar a|
       \}\le \delta$, 
       which along with the choice of $\delta$ implies 
   $f  (t, X_t^{\bar t,\bar x,\bar a }(\omega),  \alpha_t (\omega) ) \ge \epsilon>0$ and hence 
$$
   \int_{\bar t}^{\tau(\omega)}  
       e^{-\beta  (u-\bar t)}
       h        \left(u, X_u^{\bar t,\bar x, \bar a }(\omega),   \alpha_u (\omega) \right) \dif u > 0.
   $$  
  This yields a contradiction, and proves $(\mathcal L[\hat V]+r-\hat q)(t,x,a)=0$  
  for all $(t,x,\mu)\in[0,T]\times\R^n$ and $a\in \mathcal O_{\mu_{\phi}(t,x)}$. 

  Now by \eqref{eq:hjb_condition}, 
  for all $(t,x )\in[0,T]\times\R^n$,
  \begin{align*}
  (\mathcal L[\hat V] +r)(t,x,  \mu_{\phi}(t,x))
  =0,
  \quad \hat V(T,x)=g(x).
  \end{align*}
Since $V^\phi  \in C^{1,2}([0,T]\times \sR^n)$  satisfies the same PDE, 
the uniqueness of Feynman-Kac formula \citep{beck2021nonlinear} shows that   $\hat{V}(t,x)=V^\phi(t,x)$
for all $(t,x)$.
This   subsequently implies 
 $(\mathcal L [V^\phi ]+r-\hat q)(t,x, a)=0$ 
  for all $(t,x )\in[0,T]\times\R^n$ and $a\in \mathcal O_{\mu_\phi (t,x)}$. 
\end{proof}

\subsection{Proof of \Cref{thm:spg_limit}}
The following lemma shows that 
the convergence in  \eqref{eq:conv_weak}
in fact holds uniformly with respect to $u$ over a compact set.
\begin{lemma}
\label{lemma:uniform_conv}
Suppose  Assumption \ref{assum:stochastic_policy}  holds. 
For any compact subset $\mathcal K\subset \sR^d$, 
$
\lim_{\tau \searrow   0}\sup_{u\in \mathcal K}\dif_{\rm BL}          \left( \chi_\tau(a-u)\dif a, \delta_u \right)=0, 
$ 
where $\dif_{\rm BL}:\mathcal P(\sR^d)\to [0,\infty)$ is the bounded Lipschitz metric defined  by
\begin{equation}
\label{eq:BL_metric}
\dif_{\rm BL}(\mu,\nu)\coloneqq
\sup_{f\in \operatorname{Lip}_{b,1}(\sR^d)}    
       \left(\int_{\sR^d}f(a)\mu(\dif a)-\int_{\sR^d}f(a)\nu(\dif a)\right).
\end{equation}
\end{lemma}

 The following lemma shows that integrating with respect to  
 $\{\pi_{\phi,\tau}\}_{\tau>0}$
 preserves the H\"older continuity. 
\begin{lemma}
\label{lemma:holder_integral}
 Suppose \Cref{assum:regularity_limit,assum:stochastic_policy} hold.
 Let 
 $f\in C_b([0,T]\times \sR^n\times \sR^d) $ and $\phi\in \sR^k$. For each $(t,x)\in [0,T]\times \sR^n$,
 define $h_{\phi,\tau}(t,x)=\int_{\mathcal A}f(t,x,a)\pi_{\phi,\tau}(\dif a|t,x)$
  for all $\tau>0$  and
  define  $h_{\phi,0}(t,x)=f(t,x,\mu_\phi(t,x))$.
 Then
\begin{enumerate}[(1)]
    \item 
There exists a constant $C\ge 0$, depending only on $\|\mu_\phi\|_{\alpha,1}$ in \eqref{eq:policy_holder} and $C_\chi$ in \eqref{eq:lips_weak}, such that 
 for all $\tau>0$  and $\ell,\eta\in (0,1]$,
\begin{equation}
\label{eq:holder_integral}
 \sup_{t\in [0,T]}\|h_{\phi,\tau}(t,\cdot)\|_{\eta \wedge  1}
 \le C \sup_{t\in [0,T]}\|f(t,\cdot)\|_{\eta, 1}, 
 \quad
  \|h_{\phi,\tau}\|_{
 \alpha\wedge \ell, 
 \eta \wedge  1}
 \le C \|f \|_{\ell,\eta, 1},
\end{equation}
 where $\alpha \in (0,1) $
 is the H\"older component in \eqref{eq:policy_holder},
 and $\ell\wedge \eta =\min(\ell,\eta)$.
 \item  
 If $\sup_{(t,x)\in [0,T]\times \sR^n}\|f(t,x,\cdot)\|_{1}<\infty$,
 then 
 $\lim_{\tau\to 0}\|h_{\phi,\tau}-h_{\phi,0} \|_0=0$.
  \end{enumerate} 
\end{lemma}

The following lemma shows that 
the  solution to  \eqref{eq:aggregated_sde}
converges to that of \eqref{eq:state_mu} as $\tau\to 0$.

\begin{lemma}
\label{lemma:X_tau_conv}
   Suppose \Cref{assum:regularity_limit,assum:stochastic_policy} hold, and let  $\phi\in \sR^k$.
   For all  $\tau>0$, 
   the dynamics \eqref{eq:aggregated_sde} has a unique strong solution $\tilde X^{\phi,\tau}$.
   Moreover, $\lim_{\tau\to 0}\mathbb E\left[\sup_{t\in [0,T]}|\tilde X^{\phi,\tau}_t-X^\phi_t|^2 \right]=0$.
    
\end{lemma}

The following lemma establishes the  H\"older regularity  
and the convergence 
of the value function  $\tilde V^{\phi,\tau}$ and its derivatives, as  $\tau\to 0$. 
 
\begin{lemma}
\label{lemma:V_tau_holder_conv}
    Suppose \Cref{assum:regularity_limit,assum:stochastic_policy} hold, and let  $\phi\in \sR^k$.
    Then  
for each $\tau>0$, $\tilde V^{\phi,\tau}$ is the classical solution to    the following linear PDE: 
   for all $(t,x)\in [0,T]\times \sR^n$, 
\begin{align} 
\label{eq:pde_v_tau}
\begin{split} 
 & \int_{\mathcal A} \left(\mathcal L [ \tilde V^{\phi,\tau}](t,x, a) +r(t,x,a)\right)\pi_{\phi,\tau}(\dif a|t,x) =0,
\quad 
 \tilde  V^{\phi,\tau}(T,x) = g(x),
  \end{split}
\end{align}
where $\mathcal L$ is the generator defined in \eqref{eq:generator_L}. 
Moreover,
$\sup_{\tau>0}\|\tilde V^{\phi,\tau}\|_{1+\alpha/2,2+\alpha}<\infty$,
and 
the functions
$\{\tilde V^{\phi,\tau}\}_{\tau>0}$, 
$\{\partial_x \tilde V^{\phi,\tau}\}_{\tau>0}$,
and 
$\{\partial^2_{xx} \tilde V^{\phi,\tau}\}_{\tau>0}$
converge to 
$ V^{\phi}$, 
$ \partial_x V^{\phi}$,
and 
$ \partial^2_{xx} V^{\phi}$, respectively, 
uniformly on compact subsets of $[0,T]\times \sR^n$.
\end{lemma}

 The following proposition presents a  version  of SPG formula with respect to  the policy $\pi_{\phi,\tau}$.
 \begin{lemma}
 \label{lemma:spg}
      Suppose \Cref{assum:regularity_limit,assum:stochastic_policy} hold.
      For all   $(t,x)\in [0,T]\times \sR^n$, $ \phi \in \sR^k$ and $\tau>0$,
 \begin{align*}
 \begin{split}
&  \partial_\phi \tilde V^{\phi,\tau}(t,x)
\\
  &=   
   \mathbb E \left[ \int_t^T
   e^{-\beta(s-t)} 
   \partial_\phi \mu_\phi (s, \tilde X^{  \phi,\tau }_s)^\top
   \int_{\mathcal A} \partial_ a A^{\phi,\tau}(s, \tilde X^{  \phi,\tau }_s, a)
    \pi_{\phi,\tau}(\dif a| s, \tilde X^{\phi,\tau  }_s)  
   \dif s\,\bigg\vert\, \tilde X^{\phi,\tau}_t=x \right],
 \end{split}
 \end{align*}
 where 
 $A^{\phi,\tau}(t,x,a)\coloneqq  \mathcal L[ \tilde V^{\phi,\tau}](t,x, a)  +  r(t, x,   a)$. 
 \end{lemma}

 \begin{proof}[Proof of \Cref{thm:spg_limit}]
    We assume without loss of generality that $t=0$. The proof for a general $t\in (0,T]$ follows from similar arguments and   the Markov property of $\tilde X^{\phi,\tau}$ and  $ X^{\phi}$.
Define for each $\tau>0$ and $(t,x)\in [0,T]\times \sR^n$,
\begin{align*}
    F_\tau (t,x)&\coloneqq 
    e^{-\beta t} 
   \partial_\phi \mu_\phi (t, x)^\top
   \int_{\mathcal A} \partial_ a A^{\phi,\tau}(t, x, a)
    \pi_{\phi,\tau}(\dif a| t, x),
    \\
  F_0 (t,x)&\coloneqq     e^{-\beta t} 
   \partial_\phi \mu_\phi (t, x)^\top
  \partial_ a A^{\phi}(t, x, \mu_\phi(t,x)).
\end{align*}
Then for all $\tau>0$, 
\begin{align}
   & \partial_\phi \tilde V^{\phi,\tau}(0,x) - \partial_\phi  V^{\phi}(0,x)
   \nonumber
   \\
    &=\mathbb E\left[\int_0^T F_\tau(t,\tilde X^{\phi,\tau}_t)\dif t\right]-\mathbb  E\left[\int_0^T F_0(t,  X^{\phi}_t )\dif t\right]
   \nonumber
    \\
    &=\mathbb E\left[\int_0^T (F_\tau(t,\tilde X^{\phi,\tau}_t)
    -F_\tau(t,  X^{\phi}_t))\dif t\right]
   +\mathbb E\left[\int_0^T (F_\tau(t,  X^{\phi}_t) - F_0(t,  X^{\phi}_t ))\dif t\right],
\label{eq:spg-dpg}
\end{align}
where to   simplify the notation,  we have omitted  
 the explicit dependence of $\tilde X^{\phi,\tau}$ and  $ X^{\phi}$
on the initial state $x$.
We shall prove that both terms in  \eqref{eq:spg-dpg} vanish as $\tau \to 0$.

Note that 
by  \Cref{assum:regularity_limit} and \Cref{lemma:holder_integral,lemma:V_tau_holder_conv}, 
  $\sup_{t\in [0,T],\tau\ge 0}\|F_\tau(t,\cdot)\|_{\alpha}<\infty$.
  Hence there exists a constant $C\ge 0$ such that for all $\tau>0$,
  \begin{align*}
      \left|\mathbb E\left[\int_0^T (F_\tau(t,\tilde X^{\phi,\tau}_t)
    -F_\tau(t,  X^{\phi}_t))\dif t\right]\right|
    &\le  \mathbb E\left[\int_0^T |F_\tau(t,\tilde X^{\phi,\tau}_t)
    -F_\tau(t,  X^{\phi}_t)|\dif t\right]
    \\
    &\le C\mathbb E\left[\int_0^T | \tilde X^{\phi,\tau}_t 
    - X^{\phi}_t|^\alpha\dif t\right],
  \end{align*}
  which tends to zero as $\tau\to 0$ due to \Cref{lemma:X_tau_conv}.
  For the second term in \eqref{eq:spg-dpg},
using the definitions of $A^{\phi,\tau}$
and $A^{\phi}$, the boundedness and Lipschitz continuity of  
$a\mapsto (\partial_a b,\partial_a \Sigma,\partial_a r)$,
and \Cref{lemma:holder_integral,lemma:V_tau_holder_conv},
$\{F_\tau\}_{\tau >0}$
converges to 
$F_0$ as $\tau\to 0$, uniformly on compact subsets of $[0,T]\times \sR^n$. 
Since $\{F_\tau\}_{\tau \ge 0}$
are uniformly bounded, by the dominated convergence theorem,
$$
\lim_{\tau\to 0}\mathbb E\left[\int_0^T (F_\tau(t,  X^{\phi}_t) - F_0(t,  X^{\phi}_t ))\dif t\right]=0.$$
This finishes the proof.
 \end{proof}

\subsection{Proof of \Cref{prop:var_blow_up}}

\begin{proof}
    By It\^o's formula, 
    \begin{equation}
        \begin{aligned}
            &e^{-\beta h}V_\theta(\tilde{x}_{t+h})-V_\theta(\tilde{x}_t)\\
            &=\underbrace{\int_{t}^{t+h}e^{-\beta(s-t)}\left[\partial_t V_\theta(\tilde{x}_s)+\partial_xV_\theta(\tilde{x}_s)^\top b(\tilde{x}_s,a_s)+\frac{1}{2}\Tr(\partial_{xx}^2V_\theta(\tilde{x}_s)\sigma\sigma^\top(\tilde{x}_s,a_s))-\beta V_\theta(\tilde{x}_s)\right]\dif s}_{\cirone}\\
            &\qquad +\underbrace{\int_{t}^{t+h}e^{-\beta(s-t)}\partial_x V_\theta(\tilde{x}_s)^\top\sigma(\tilde{x}_s,a_s)\dif W_s}_{\cirtwo}.
        \end{aligned}
    \end{equation}
    Note that the last term is a martingale and thus vanishes after taking expectation. Therefore the semi-gradient can be rewritten as 
    \begin{equation}
        \begin{aligned}
            G_{\theta,h}
            =\E\left[\partial_\theta V_\theta(\tilde{x}_t)\left(-\cirone\cdot \frac{1}{h}+(A_\psi(\tilde{x}_t,a_t)-r_t)\right)\right].
        \end{aligned}
    \end{equation}
    When the discretization step $h$ goes to zero, the integral $\cirone$ admits a first-order expansion, hence
    \begin{equation}
        \lim_{h\to 0}G_{\theta,h}=\E\left[\partial_\theta V_\theta(\tilde{x}_t)\left(A_\psi(\tilde{x}_t,a_t)-\partial_t V_\theta(\tilde{x}_t)-H(\tilde{x}_t,a_t,\partial_x V_\theta(\tilde{x}_t),\partial_{xx}^2V_\theta(\tilde{x}_t))+\beta V_\theta(\tilde{x}_t)\right)\right].
    \end{equation}
    Similarly we have
    \begin{equation}
        \lim_{h\to 0}G_{\psi,h}=\E\left[\partial_\psi A_\psi(\tilde{x}_t,a_t)\left(A_\psi(\tilde{x}_t,a_t)-\partial_t V_\theta(\tilde{x}_t)-H(\tilde{x}_t,a_t,\partial_x V_\theta(\tilde{x}_t),\partial_{xx}^2V_\theta(\tilde{x}_t))+\beta V_\theta(\tilde{x}_t)\right)\right].
    \end{equation}
    On the other hand, consider the conditional variance of stochastic gradient:
    \begin{equation}\label{eq:grad_cond_var}
        \begin{aligned}
            \Var(g_{\theta,h}\mid\mathcal{F}_t)=\frac{1}{h^2}\partial_\theta V_\theta(\tilde{x}_t)\partial_\theta V_\theta(\tilde{x}_t)^\top\Var(e^{-\beta h}V_\theta(\tilde{x}_{t+h})-V_\theta(\tilde{x}_{t})\mid\mathcal{F}_t).
        \end{aligned}
    \end{equation}
    Note that
    \begin{equation}
        \begin{aligned}
            \E[(e^{-\beta h}V_\theta(\tilde{x}_{t+h})-V_\theta(\tilde{x}_{t}))^2\mid\mathcal{F}_t]
            &= \E[\cirone^2+2\cdot\cirone\cdot\cirtwo+\cirtwo^2\mid\mathcal{F}_t],
        \end{aligned}
    \end{equation}
    and $\E[e^{-\beta h}V_\theta(\tilde{x}_{t+h})-V_\theta(\tilde{x}_{t})\mid\mathcal{F}_t]=\E[\cirone\mid\mathcal{F}_t]$. This yields
    \begin{equation}
        \Var(e^{-\beta h}V_\theta(\tilde{x}_{t+h})-V_\theta(\tilde{x}_{t})\mid\mathcal{F}_t)=\Var(\cirone\mid\mathcal{F}_t)+\E\left[\cirtwo^2+2\cdot\cirone\cdot\cirtwo\mid\mathcal{F}_t\right]\geq\E\left[\cirtwo^2+2\cdot\cirone\cdot\cirtwo\mid\mathcal{F}_t\right]
    \end{equation}
    According to It\^o isometry,
    \begin{equation}
        \E[\cirone^2\mid\mathcal{F}_t]=\mathcal{O}(h^2),
    \end{equation}
    \begin{equation}
        \E[\cirtwo^2\mid\mathcal{F}_t]=\E\left[\int_{t}^{t+h}e^{-2\beta(s-t)}\|\partial_x V_\theta(\tilde{x}_s)^\top\sigma(\tilde{x}_s,a_s)\|^2\dif s\bigg| \tilde{x}_t\right]=\mathcal{O}(h),
    \end{equation}
    and the cross term can be controlled by Cauchy-Schwarz:
    \begin{equation}
        \E[|\cirone\cdot\cirtwo|\mid\tilde{x}_{t}]\leq (\E[\cirone^2\mid\tilde{x}_{t}])^{\frac{1}{2}}\cdot(\E[\cirtwo^2\mid\tilde{x}_{t}])^{\frac{1}{2}}=\mathcal{O}(h^{\frac{3}{2}})
    \end{equation}
    These estimates show that, as $h\to 0$, the leading contribution to the variance comes from the stochastic integral term $\cirtwo$.
    As a result, by combining Fatou's Lemma and \eqref{eq:grad_cond_var}, we conclude that
    \begin{equation}
        \begin{aligned}
            \lim_{h\to 0}h\cdot\Var(g_{\theta,h})
            &\geq \lim_{h\to 0}h\cdot\E[\Var(g_{\theta,h}\mid\mathcal{F}_t)] \\
            &\geq\E\left[\lim_{h\to 0}[h\cdot\Var(g_{\theta,h}\mid\mathcal{F}_t)]\right] \\
            &\geq \E\left[\partial_\theta V_\theta(\tilde{x}_t)\partial_\theta V_\theta(\tilde{x}_t)^\top\|\partial_x V_\theta(\tilde{x}_t)^\top\sigma(\tilde{x}_t,a_t)\|^2\right].
        \end{aligned}
    \end{equation}
\end{proof}

\subsection{Proof of \Cref{prop:var_normal}}

\begin{proof}
    We begin by recalling that, for any horizon $Lh$, It\^o's formula yields, 
    \begin{equation}
        \begin{aligned}
            &e^{-\beta Lh}V_\theta(\tilde{x}_{t+Lh})-V_\theta(\tilde{x}_t)\\
            &=\underbrace{\int_{t}^{t+Lh}e^{-\beta(s-t)}\left[\partial_t V_\theta(\tilde{x}_s)+\partial_xV_\theta(\tilde{x}_s)^\top b(\tilde{x}_s,a_s)+\frac{1}{2}\Tr(\partial_{xx}^2V_\theta(\tilde{x}_s)\sigma\sigma^\top(\tilde{x}_s,a_s))-\beta V_\theta(\tilde{x}_s)\right]\dif s}_{\cirthree}\\
            &\qquad +\underbrace{\int_{t}^{t+Lh}e^{-\beta(s-t)}\partial_x V_\theta(\tilde{x}_s)^\top\sigma(\tilde{x}_s,a_s)\dif W_s}_{\cirfour}.
        \end{aligned}
    \end{equation}
    Now consider the case where $Lh \equiv \delta > 0$ is fixed while $h\to 0$. 
    In this regime, the estimator $G_{\theta,h,\delta/h}$ can be expressed as
    \begin{equation}
        \begin{aligned}
        \lim_{h\to 0}G_{\theta,h,\frac{\delta}{h}}
        &=\E\left[\partial_\theta V_\theta(\tilde{x}_t)\left(V_\theta(\tilde{x}_t)-\int_t^{t+\delta}e^{-\beta (s-t)}[r_s-q_\psi(\tilde{x}_s,a_s)]\dif s-e^{-\beta \delta}V_\theta(\tilde{x}_{t+\delta})\right)\right]=\Theta(1)\\
        &=\E\left[\partial_\theta V_\theta(\tilde{x}_t)\left(\int_{t}^{t+\delta}\left[q_\psi(\tilde{x}_s,a_s)-\partial_t V_\theta(\tilde{x}_s)-H(\tilde{x}_s,a_s,\partial_x V_\theta(\tilde{x}_s),\partial_{xx}^2V_\theta(\tilde{x}_s))+\beta V_\theta(\tilde{x}_s)\right]\dif s\right)\right]\\
        &=\Theta(1).
        \end{aligned}
    \end{equation}
    The integral is taken over a fixed interval of length $\delta$, and thus this expression is bounded and will not vanish.
    
    We next turn to the variance. Expanding the definition of $g_{\theta,h,L}$ and using Jensen’s inequality, we obtain
    \begin{equation}
        \begin{aligned}
        &\Var(g_{\theta,h,L})
        \\
        &\leq 2\E\left[\partial_\theta V_\theta(\tilde{x}_t)\partial_\theta V_\theta(\tilde{x}_t)^\top\left((e^{-\beta Lh}V_\theta(\tilde{x}_{t+Lh})-V_\theta(\tilde{x}_t))^2+(\sum_{l=0}^{L-1}e^{-\beta lh}[r_{t+lh}-q_\psi(\tilde{x}_{t+lh},a_{t+lh})]h)^2\right)\right]\\
        &=\mathcal{O}(1).
        \end{aligned}
    \end{equation}
    This is because all terms are bounded.
\end{proof}

\section{Proofs of Technical Lemmas}

\subsection{Proof of \Cref{lemma:performance_difference}}

\begin{proof}
Observe that under \Cref{assum:wp}, for each $\phi\in \sR^k$,
and  $(t,x) \in [0,T]\times \sR^n$,
\begin{align}
V^\phi(t,x ) \coloneqq \mathbb E \left[
\int_t^T 
e^{-\beta (s-t)}
r(s,X_s^{t,x,\phi},\mu_\phi(s,X_s^{t,x,\phi}))\dif t+e^{-\beta (T-s)}g(X_T^{t,x,\phi})
\right],
\end{align}
where $(X_s^{t,x,\phi})_{s\ge t}$ satisfies for all $s\in [t,T]$, 
\begin{equation}
    \dif X_s=b(s,X_s,\mu_\phi(s, X_s))\dif s+\sigma(s,X_s,\mu_\phi(s,X_s))\dif W_s, 
    \quad X_t=x. 
\end{equation}
Fix $\phi'\in \sR^d$. 
Denote by $X^{\phi}= X^{t,x,\phi}$ and $X^{\phi'}= X^{t,x,\phi'}$ for simplicity.
Then 
    \begin{align}
    \label{eq:performance_difference1}
    \begin{split}
       & V^{\phi'} (t, x )-V^{\phi}(t,x) 
       \\
        & =
         \mathbb E \left[
\int_t^T e^{-\beta(s-t)} r (s, X^{  \phi' }_s, \mu_{\phi'}(s,X_s^{\phi'}) ) \dif s  \right]+ 
   e^{-\beta(T-t)}   \mathbb E \left[g  (X^{  \phi' }_T )\right]
   -V^{\phi}(t,x) 
\\  
&=  \mathbb E \left[
\int_t^T e^{-\beta(s-t)} r (s, X^{  \phi' }_s, \mu_{\phi'}(s,X_s^{\phi'}) ) \dif s  \right]
+
   \mathbb E \left[e^{-\beta(T-t)}  V^{\phi} (T, X^{\phi'  }_T  )\right]
-V^{\phi}(t,X^{x,\phi'  }_t),
\end{split}
    \end{align}
  where    the last identity used the fact that $V^\phi(T,x)=g(x) $ and
  $ {X^{\phi' }_t} =x$.
As    $V^\phi \in C^{1,2}([0,T]\times \sR^n)$,
  applying   It\^{o}'s formula
to 
$s\mapsto e^{-\beta (s-t)}  V^\phi(s, X^{\phi' }_s)$ yields
\begin{align*}
   & \mathbb E \left[e^{-\beta(T-t)}  V^{\phi} (T, X^{\phi'  }_T  )\right]
-V^{\phi}(t,X^{\phi'  }_t)
\\
 &   = 
     \mathbb E \left[ \int_t^T
   e^{-\beta(s-t)} \mathcal L [V^\phi] (s, X^{\phi'}_s, \mu_{\phi'}(s, X^{\phi'}_s))\dif s \right]
 \\
 &  
 = 
     \mathbb E \left[ \int_t^T
   e^{-\beta(s-t)} 
   \left(
\left(\mathcal L [V^\phi] (s, y, \mu_{\phi'}(s, y)) 
   -\mathcal L [V^\phi] (s, y, \mu_{\phi}(s, y))
   \right)
   \bigg\vert_{y=X^{\phi'}_s}
   +\mathcal L [V^\phi] (s, X^{\phi'}_s, \mu_{\phi}(s, X^{\phi'}_s))
   \right) \dif s \right]
   \\
   &  
    = 
     \mathbb E  \left[ \int_t^T
   e^{-\beta(s-t)} 
     \left(
\left(\mathcal L [V^\phi] (s, y, \mu_{\phi'}(s, y)) 
   -\mathcal L [V^\phi] (s, y, \mu_{\phi}(s, y))
   \right)
   \bigg\vert_{y=X^{\phi'}_s}
   - r(s,X^{\phi'}_s,\mu_\phi(s,X^{\phi'}_s))
   \right)
   \dif s \right], 
\end{align*}
where the last identity used the PDE \eqref{eq:pde_v}. 
This along with 
\eqref{eq:performance_difference1}
proves the desired result. 
\end{proof}

\subsection{Proof of \Cref{lemma:uniform_conv}}
 \begin{proof}
 Let $\mathcal P(\sR^d)$
 be the metric space of probability measures on $\sR^d$,  equipped 
 with the metric $\dif_{\rm BL}$.
 Consider the class of continuous functions 
  $\{f_\tau\}_{\tau\ge 0}\subset C(\sR^d,\mathcal P(\sR^d))$,
 where for each $u\in\sR^d$,
 $f_\tau (u)\coloneqq \chi_\tau(a-u)\dif a$  for   $\tau>0$,   and $f_0(u)\coloneqq \delta_u$.
 Note that
 for each $u\in \sR^d$, 
  the fact that   $\lim_{\tau \to 0}\dif_{\rm BL} ( f_\tau (u), f_0(u))=0$ in  \eqref{eq:conv_weak} implies that 
 $\{f_\tau (u)\}_{\tau\ge 0}$
 has a compact closure in $\mathcal P(\sR^d)$. Moreover, Condition \eqref{eq:conv_weak} implies that 
 $\dif_{\rm BL} ( f_\tau (u),  f_\tau (v))\le |u-v|$ for all $\tau>0$,
 and hence $\{f_\tau\}_{\tau\ge 0}$ is equicontinuous under $\dif_{\rm BL}$. 
 Thus by Ascoli's theorem \cite[Theorem 47.1]{munkres2000topology}, 
 $\{f_\tau\}_{\tau\ge 0}$ is contained in a compact subset of $C(\sR^d,\mathcal P(\sR^d))$ with respect  to  the topology of compact convergence. This along with the pointwise convergence of 
 $\{f_\tau\}_{\tau> 0}$ yields the desired conclusion.  
\end{proof}

\subsection{Proof of \Cref{lemma:holder_integral}}

\begin{proof}
    Observe that 
    $\|h\|_0  \le \|f\|_0$. 
    To show the H\"older continuity, 
    let  $t\in [0,T]$ and $x,x'\in \sR^n$,   by \Cref{assum:regularity_limit,assum:stochastic_policy},
    \begin{align*}
      &  |h_{\phi,\tau}(t,x)-h_{\phi,\tau}(t,x')|
      \\
        &=\left|\int_{\mathcal A}f(t,x,a)\pi_{\phi,\tau}(\dif a|t,x)-\int_{\mathcal A}f(t,x,a)\pi_{\phi,\tau}(\dif a|t,x')\right|
        \\
        &\quad 
        +\left|
        \int_{\mathcal A}f(t,x,a)\pi_{\phi,\tau}(\dif a|t,x')-\int_{\mathcal A}f(t,x',a)\pi_{\phi,\tau}(\dif a|t,x')\right|
        \\
        &\le \sup_{(t,x)\in [0,T]\times \sR^n}\|f(t,x,\cdot)\|_{1} C_{\chi}|\mu_\phi(t,x)-\mu_\phi(t,x')|
        +\sup_{(t,a)\in [0,T]\times \sR^d}\|f(t, \cdot,a)\|_{\eta }|x-x'|^\eta
        \\
        &\le C_{\chi} \sup_{(t,x)\in [0,T]\times \sR^n}\|f(t,x,\cdot)\|_{1} \|\mu_\phi(t,\cdot)\|_{1}
        |x-x'|
        +\sup_{(t,a)\in [0,T]\times \sR^d}\|f(t, \cdot,a)\|_{\eta }|x-x'|^\eta,
    \end{align*}
    where the first inequality used \eqref{eq:density} and \eqref{eq:lips_weak}.
    Similarly, for all $t,t'\in [0,T]$ and $x\in \sR^n$,
    $ |h_{\phi,\tau}(t,x)-h_{\phi,\tau}(t',x)|
      \le C_{\chi} \sup_{(t,x)\in [0,T]\times \sR^n}\|f(t,x,\cdot)\|_{1} \|\mu_\phi(\cdot,x)\|_{\alpha}
        |x-x'|^\alpha
        +\sup_{(x,a)\in \sR^n\times \sR^d}\|f( \cdot,x, a)\|_{\ell }|x-x'|^\ell.$
    This proves the desired estimate \eqref{eq:holder_integral}. 
    
    Finally,     for all  $(t,x)\in [0,T]\times \sR^n$,
        $\left|h_{\phi,\tau}(t,x)-h_{\phi,0}(t,x)\right|
        \le \sup_{(t,x)\in [0,T]\times \sR^n}\|f(t,x,\cdot)\|_{1} \dif_{\rm BL}(\chi_\tau(a-\mu_\phi(t,x))\dif a, \delta_{\mu_\phi(t,x)}),$
    which along with 
    the uniform boundedness of  $\mu_\phi$ and  \Cref{lemma:uniform_conv} implies the uniform convergence of $\{h_{\phi,\tau}\}_{\tau>0}$.
\end{proof}

\subsection{Proof of \Cref{lemma:X_tau_conv}}

\begin{proof}
By \cite[Theorem 6.2]{higham2008functions},
the     principal square root function is Lipschitz continuous on the space of 
positive definite matrices whose minimum eigenvalues are uniformly bounded away from zero.  
    As $\|b\|_{\alpha,1,1}+\|\Sigma\|_{\alpha,1,1}<\infty$ and $\Sigma$ is uniformly elliptic, 
    by  \Cref{lemma:holder_integral}, 
    $\sup_{\tau >0}(\|\tilde 
    b^{\pi_{\phi,\tau}}\|_{\alpha, 1}
    +\|\tilde \sigma^{\pi_{\phi,\tau}}\|_{\alpha, 1})
    <\infty$.  Thus by \cite[Theorem 3.3.1]{zhang2017backward},
    \eqref{eq:aggregated_sde} admits a unique square-integrable solution. 
Moreover, by \cite[Theorem 3.2.4]{zhang2017backward}, 
there exists a constant $C>0$ such that 
for all $\tau>0$, 
\begin{align*}
   \mathbb E\left[\sup_{t\in [0,T]}|\tilde X^{\phi,\tau}_t-X^\phi_t|^2 \right]
   \le C\left[\left(\int_0^T |\Delta b_\tau(t,X^\phi_t)|\dif t\right)^2+\int_0^T |\Delta \sigma_\tau(t,X^\phi_t)|^2\dif t\right],
\end{align*}
where $
    \Delta b_\tau (t,x)=\tilde b^{\pi_{\phi,\tau}}(t,x)-b(t,x,\mu_\phi(t,x)),
    \quad 
    \Delta \sigma_\tau (t,x)=\tilde \sigma^{\pi_{\phi,\tau}}(t,x)-\sigma(t,x,\mu_\phi(t,x)).$
By   \Cref{lemma:holder_integral},
$\lim_{\tau \to 0}\|\Delta b_\tau\|_0=0$,
and 
by the Lipschitz continuity of the principal square root function, $
    \|\Delta \sigma_\tau\|_0
    \le C\sup_{(t,x)\in [0,T]\times \sR^n}\left|\int_A \Sigma(t,x,a)\pi_{\phi,\tau}(\dif a|t,x)-\Sigma(t,x,\mu_\phi(t,x))\right|,$
which converges to zero as $\tau\to 0$. 
This proves the desired convergence of 
$(\tilde X^{\phi,\tau})_{\tau>0}$.
\end{proof}

\subsection{Proof of \Cref{lemma:V_tau_holder_conv}}

\begin{proof}
  The  PDE \eqref{eq:pde_v_tau} can be equivalently written as 
  \begin{align*}
\begin{split}
\partial_t \tilde V^{\phi,\tau}(t,x)
     -\beta  \tilde V^{\phi,\tau}(t,x)
     +  \tilde b_\tau(t,x)^\top \partial_x \tilde V^{\phi,\tau}(t,x) +\frac{1}{2} \textrm{Tr}(\tilde \Sigma_\tau (t,x)     
     \partial^2_{xx} \tilde V^{\phi,\tau}(t,x))
     +\tilde r_\tau(t,x)=0,
     \end{split}
\end{align*}
where for each $\varphi\in \{b,\Sigma,r\}$,
$\tilde \varphi_\tau(t,x)\coloneqq \int_{\mathcal A} \varphi(t,x,a) \pi_{\phi,\tau}(\dif a|t,x) $.
By \Cref{assum:regularity_limit,assum:stochastic_policy} and \Cref{lemma:holder_integral}, 
$\sup_{\tau>0}(\|\tilde b_\tau\|_{\alpha/2,\alpha}
+\|\tilde \Sigma_\tau\|_{\alpha/2,\alpha}
+\|\tilde r_\tau\|_{\alpha/2,\alpha})<\infty$,
and $\xi^\top\tilde\Sigma_\tau(t,x)\xi\ge \kappa |\xi|^2$ for all $\xi\in \sR^n$. 
Hence
by Schauder’s estimate \cite[Theorem 9.2.3]{krylov1996lectures}, 
\eqref{eq:pde_v_tau} admits a unique strong solution in $C^{1+\alpha/2,2+\alpha}_b([0,T]\times \sR^n)$ and its H\"older norm $\|\cdot\|_{1+\alpha/2,2+\alpha}$ is independent of $\tau$.
By It\^o's formula, the classical solution to \eqref{eq:pde_v_tau} coincides with  $\tilde V^{\phi,\tau}$  
defined in \eqref{eq:value_aggregate}. 

It remains to prove the convergence of 
$\{\tilde V^{\phi,\tau}\}_{\tau>0}$ and their derivatives. 
By \eqref{eq:pde_v} and \eqref{eq:pde_v_tau}, $U\coloneqq \tilde V^{\phi,\tau}-V^\phi$ is the classical solution to the following PDE: 
\begin{align} 
\label{eq:pde_v_difference}
\begin{split} 
 & \int_{\mathcal A}  \mathcal L [ U](t,x, a) \pi_{\phi,\tau}(\dif a|t,x)
 +F_\tau(t,x)=0,
\quad 
U(T,x) = 0,
  \end{split}
\end{align}
where \begin{align*}
F_\tau(t,x)&\coloneqq 
\frac{1}{2} \textrm{Tr}((\tilde \Sigma_\tau (t,x) 
-\Sigma(t,x,\mu_\phi(t,x))
     \partial^2_{xx}  V^{\phi }(t,x)) +(\tilde b_\tau(t,x)
     - b(t,x,\mu_\phi(t,x)))^\top \partial_x   V^{\phi}(t,x)
\\
&\quad +\tilde r_\tau(t,x) -r(t,x,\mu_\phi(t,x)).
    \end{align*}
By \Cref{lemma:holder_integral}, 
$\lim_{\tau \to 0}\|F_\tau\|_0=0$.
Then by the   maximum principle for linear parabolic PDEs \citep[Corollary 8.1.5]{krylov1996lectures}, $ \lim_{\tau\to 0}\|U\|_{0}=0$,
which implies $\{\tilde V^{\phi,\tau}\}_{\tau>0}$  converges to $V^\phi$, uniformly on $[0,T]\times \sR^n$. 
Since $\sup_{\tau>0}\|\tilde V^{\phi,\tau}\|_{1+\alpha/2,2+\alpha}<\infty$, by the Arzel\`a–Ascoli theorem, 
$\{\partial_{x} \tilde V^{\phi,\tau}\}_{\tau>0}$ and $\{\partial^2_{xx} \tilde V^{\phi,\tau}\}_{\tau>0}$
converge to 
$ \partial_x V^{\phi}$,
and 
$ \partial^2_{xx} V^{\phi}$, respectively, 
uniformly on compact subsets of $[0,T]\times \sR^n$.
\end{proof}   

\subsection{Proof of \Cref{lemma:spg}}

\begin{proof}
    Using an argument analogous to that in the proof of \Cref{prop:pg}, one can prove that 
    \begin{align}
    \label{eq:spg_v1}
        \partial_\phi \tilde V^{\phi,\tau}(t,x)
   =  
     \mathbb E \left[ \int_t^T
   e^{-\beta(s-t)} 
   \int_{\mathcal A} A^{\phi,\tau}(s, \tilde X^{  \phi,\tau }_s, a)
   \partial_\phi   \pi_{\phi,\tau}(s, \tilde X^{\phi,\tau  }_s,a)  \dif a 
   \dif s\,\bigg\vert\, \tilde X^{\phi,\tau}_t=x \right].
    \end{align}
    This  version of the PG formula for stochastic policies has also been given in \cite[Equation (14)]{jia2022policy_grad}.
    
    We now rewrite the integral with respect to $\dif a$ in \eqref{eq:spg_v1} using the integration-by-parts formula. 
    Fix $(t,x)\in [0,T]\times \sR^n$.
    By the expression \eqref{eq:density} of $\pi_{\phi,\tau}$ and $\mathcal A=\sR^d$, for all $j=1,\dots, k$,
    \begin{align}
    \label{eq:spg_v2_step1}
    \begin{split}
         \int_{\mathcal A} A^{\phi,\tau}(t, x, a)
   \partial_{\phi_j}   \pi_{\phi,\tau}(t, x,a)  \dif a 
   &=   -\partial_{\phi_j}   \mu_\phi(t,x)  \int_{\mathcal A} A^{\phi,\tau}(t, x, a)
   \chi_{\tau}'(a-\mu_\phi(t,x))
    \dif a 
    \\
    &= -\partial_{\phi_j}   \mu_\phi(t,x)  \int_{\mathcal A} A^{\phi,\tau}(t, x, a+\mu_\phi(t,x))
   \chi_{\tau}'(a)
    \dif a,
    \end{split}
    \end{align}
    where the last integral is well-defined    since $\chi'_\tau $ is integrable and $A^{\phi,\tau}$ is bounded by 
     \Cref{lemma:V_tau_holder_conv}.  
   For each $R>0$ and $i=1,\ldots, k$, by the integration-by-parts formula,
    \begin{align*}
    &    
    \int_{B_R} A^{\phi,\tau}(t, x, a+\mu_\phi(t,x))
   \partial_{a_i}\chi_{\tau}(a)
    \dif a
    \\
    &=
    - \int_{B_R} \partial_{a_i} A^{\phi,\tau}(t, x, a+\mu_\phi(t,x))
   \chi_{\tau}(a)
    \dif a+\int_{\partial B_R}A^{\phi,\tau}(t, x, a+\mu_\phi(t,x))
   \chi_{\tau}(a) \nu_i(a)\dif S,  
    \end{align*}
   where $\nu_i(a)=a_i/R $ is the $i$-th component of the outward unit normal vector field on $\partial B_R$. Letting $R\to \infty$ and using \Cref{assum:stochastic_policy} \Cref{item:chi_regularity} 
   and the boundedness of $A^{\phi,\tau}$ and $\partial_a A^{\phi,\tau}$
   yield that 
   \begin{align*}
  &  \int_{\mathcal A} A^{\phi,\tau}(t, x, a+\mu_\phi(t,x))
   \chi_{\tau}'(a)
    \dif a= -   \int_{\mathcal A} \partial_{a} A^{\phi,\tau}(t, x, a+\mu_\phi(t,x))
   \chi_{\tau}(a)
    \dif a 
    \\
    &
    =-   \int_{\mathcal A} \partial_{a} A^{\phi,\tau}(t, x, a)
   \chi_{\tau}(a-\mu_\phi(t,x))
    \dif a
    =-   \int_{\mathcal A} \partial_{a} A^{\phi,\tau}(t, x, a)
   \pi_{\phi,\tau}
    (\dif a|t,x).  
   \end{align*}
   This with 
   \eqref{eq:spg_v1} and 
   \eqref{eq:spg_v2_step1} yields the desired identity. 
\end{proof}

\section{Conclusion}
In this paper, we investigate deterministic policy gradient methods to achieve stability and efficiency for deep
RL in continuous-time environments, bridging the gap between discrete and continuous time algorithms.
We develop a rigorous mathematical framework and provide a model-free characterization for DPG.
We further theoretically demonstrate the issues of standard one-step TD method in continuous-time regime for the first time.
All our theoretical results are verified through extensive experiments.
We hope this work can motivate future researches on continuous-time RL.

\bibliography{ref}
\bibliographystyle{plain}

\end{document}

%% file: math_commands.tex
\usepackage{amsmath,amsfonts,bm}

\def\eqref#1{equation~\ref{#1}}
\def\1{\bm{1}}

\DeclareMathAlphabet{\mathsfit}{\encodingdefault}{\sfdefault}{m}{sl}
\SetMathAlphabet{\mathsfit}{bold}{\encodingdefault}{\sfdefault}{bx}{n}

\def\sN{{\mathbb{N}}}

\def\sP{{\mathbb{P}}}

\def\sR{{\mathbb{R}}}

\newcommand{\E}{\mathbb{E}}

\newcommand{\R}{\mathbb{R}}

\newcommand{\Var}{\mathrm{Var}}

\DeclareMathOperator{\Tr}{Tr}

%% file: def.tex
\newcommand{\dif}{\mathrm{d}}

\newtheorem{thm}{Theorem}[section]
\newtheorem{rmk}{Remark}
\newtheorem{Assumption}{Assumption}
\newtheorem{lemma}[thm]{Lemma}

\newtheorem{prop}[thm]{Proposition}
\theoremstyle{definition}

\newtheorem{example}{Example}
\crefname{equation}{eq.}{eqs.}
\Crefname{equation}{Equation}{Equations}
\crefname{Assumption}{Assumption}{Assumptions}
\Crefname{Assumption}{Assumption}{Assumptions}
\crefname{thm}{Theorem}{Theorems}
\Crefname{thm}{Theorem}{Theorems}
\crefname{prop}{Proposition}{Propositions}
\Crefname{prop}{Proposition}{Propositions}
\crefname{definition}{def.}{defs.}
\Crefname{definition}{Def.}{Defs.}
\crefname{algorithm}{alg.}{algs.}
\Crefname{algorithm}{Alg.}{Algs.}
\crefname{figure}{Figure}{Figures}
\Crefname{figure}{Figure}{Figures}
\crefname{section}{Section}{Sections}
\Crefname{section}{Section}{Sections}